\documentclass[11pt,letterpaper]{article}
\usepackage[margin=1in]{geometry}
\usepackage[T1]{fontenc}
\usepackage{lmodern,microtype}
\usepackage{amsmath,amssymb,amsthm}
\usepackage{array}
\usepackage[hidelinks]{hyperref}
\newtheorem{theorem}{Theorem}
\newtheorem{lemma}[theorem]{Lemma}
\newtheorem{proposition}[theorem]{Proposition}
\newtheorem{corollary}[theorem]{Corollary}
\theoremstyle{remark}

\newcommand{\R}{\mathbb R}
\newcommand{\II}{\mathrm{II}}
\newcommand{\norm}[1]{\left\|#1\right\|}
\DeclareMathOperator{\rank}{rank}
\title{Gauss--Newton Accuracy and Indefinite Hessians:\\
Uniform Coexistence in Low-Cost Sets}
\author{Kihun Rhee\\\small\texttt{rheekh00@snu.ac.kr}
\and Hanjoon Byun\\\small\texttt{byunhanjoon@snu.ac.kr}}
\date{}
\hypersetup{pdftitle={Gauss-Newton Accuracy and Indefinite Hessians: Uniform Coexistence in Low-Cost Sets},pdfauthor={Kihun Rhee; Hanjoon Byun}}
\begin{document}
\maketitle
\begin{abstract}
We study the accuracy of Gauss--Newton curvature in ridge-regularized
nonlinear least squares. Under local regularity and persistence of
level-set curvature magnitude along an exact-fit section, we prove uniform
coexistence of two curvature regimes. Global minimizers exist, and every global minimizer has relative
Hessian error below $(1+\sqrt2)/8$, while the same low-cost set contains
a point with an indefinite Hessian and relative error at least $15/8$.
One positive ridge cap works for all independent center and label
perturbations in fixed neighborhoods and every positive ridge weight
up to the cap. These neighborhoods do not shrink as the ridge weight
tends to zero. A pointwise certificate based on the current prediction
level set controls the normal, mixed, and tangent parts of the Hessian
correction. We prove a sharp relative-error bound over the stated
pointwise class when the prediction map and ridge vary.
Analytic examples describe the roles of output
alignment, curvature orientation, and persistence. A separate structural
result gives full Jacobian row rank throughout low-cost sets and exact
interpolation near a rank-deficient reference.

\end{abstract}
\section{Introduction}\label{sec:intro}
Gauss--Newton approximates the curvature of a least-squares objective
using the prediction Jacobian. It is a standard tool in nonlinear
least squares~\cite{eriksson,gratton2007} and in Hessian-free optimization
of neural networks~\cite[Section 4.2]{martens2010}.
A small objective value controls the residual. It is therefore natural
to ask whether it also makes the Gauss--Newton approximation accurate
throughout a low-cost set.

The obstruction lies in directions that leave predictions unchanged
to first order. The prediction-Jacobian contribution vanishes in these
directions, so ridge regularization supplies all of the Gauss--Newton
curvature. The full Hessian also contains a residual-weighted
second-derivative term. This classical correction can be large relative
to the ridge even when the residual is small~\cite{eriksson,gratton2007,sjd}.
The geometry of the current prediction level set describes its action
in these directions.

Let $G:\R^p\to\R^n$ be a smooth prediction map. For labels $Y$,
ridge center $d$, and weight $\lambda>0$, write
\[
 Q(\theta)=\tfrac12\|G(\theta)-Y\|^2+
              \tfrac\lambda2\|\theta-d\|^2,
 \qquad P=DG(\theta)^TDG(\theta)+\lambda I.
\]
With $H=\nabla^2Q(\theta)$, define the relative curvature error
\begin{equation}\label{eq:rho-intro}
 \rho(\theta)=\|P^{-1/2}(H-P)P^{-1/2}\|_2
 =\sup_{h\ne0}\frac{|h^T(H-P)h|}{h^TPh}.
\end{equation}
Thus small $\rho$ means that $P$ and $H$ assign similar curvature
to every direction at the point under consideration.
We study this comparison on sets of the form
$\{Q\le C\lambda\varepsilon^2\}$.

\pagebreak[3]
\begin{enumerate}
\item \textbf{Uniform coexistence.}
Our main result is Theorem~\ref{thm:robust}.
Assume common local $C^{2,1}$ bounds and persistence of curvature
magnitude along an exact-fit section. We fix a positive cap $W$
before varying the center and labels. For every allowed independent
pair and every $0<\lambda\le W$, global minimizers exist and all
satisfy $\rho<c_*=(1+\sqrt2)/8$. The same low-cost set contains
a point with an indefinite Hessian and $\rho\ge15/8$.
The center and label neighborhoods remain fixed as $\lambda\to0$.
The minimizers and indefinite point may vary with the objective.

\item \textbf{A pointwise accuracy certificate.}
Theorem~\ref{thm:point} and Corollary~\ref{cor:relative-gate}
use cost, rank, derivative bounds, and current level-set curvature
to control the Hessian correction. The proof treats its normal,
mixed, and tangent blocks together. The resulting constant $c_*$
is a sharp supremum over the stated pointwise class when the map
and ridge vary. A fixed two-output example explains which
alignment information is lost by aggregate bounds.

\item \textbf{Rank and interpolation near a singular reference.}
Section~\ref{sec:bridge} gives derivative conditions under which
every low-cost point has full Jacobian row rank and every allowed
refitting problem has a low-cost exact interpolator.
This structural result is separate: applying the coexistence theorem
still requires its common regularity and curvature-persistence assumptions.
\end{enumerate}

Earlier results give Hessian bounds near solution manifolds and
nearby nonconvexity~\cite{masiha,liuzhubelkin,rebjock,sjd}.
Our conclusion places accurate global minimizers and an indefinite
point in the same low-cost set for each member of an entire family
of objectives. The common ridge cap and fixed perturbation
neighborhoods make this a uniform statement.

We state the coexistence theorem first, then develop the pointwise
certificate used in its proof. The later sections give analytic
comparisons and the structural result. Appendices A--D contain
the detailed proofs and further examples.

\section{Objective and level-set geometry}\label{sec:setup}
For $G:\R^p\to\R^n$, $p\ge n$, define the raw Euclidean objective
\begin{equation}\label{eq:objective}
 Q(\theta)=\tfrac12\norm{G(\theta)-Y}^2+
                        \tfrac\lambda2\norm{\theta-d}^2.
\end{equation}
Write $r=G(\theta)-Y$, $J=DG(\theta)$ and $b=\nabla Q(\theta)$.
Its complete Hessian is
\begin{align}\label{eq:hessian}
 H&=J^TJ+E+\lambda I,\qquad
 E=\sum_i r_iD^2G_i(\theta),\\
 b&=J^Tr+\lambda(\theta-d).\nonumber
\end{align}
This is the residual-inclusive formula in Eriksson et
al.~\cite{eriksson}, Theorem 4.1, (26)--(27).

At a full-row-rank point, let $T=\ker J$ and $N=T^\perp$. Define
\begin{align}\label{eq:II}
 \II(v,w)&=-J^\dagger D^2G[v,w]\quad(v,w\in T),\\
 K&=\norm{D^2G|_{T\times T}},\quad B=\norm{\II}.\nonumber
\end{align}
These are bilinear operator norms; set $K=B=0$ if $T=\{0\}$.
The tensor $\II$ is the extrinsic second fundamental form of the
regular prediction level set $G^{-1}(G(\theta))$ in the fixed raw
Euclidean parameter metric \cite[Lemma 2.2.1]{feehan}. That current
level set equals the target fiber $G^{-1}(Y)$ only at an exact fit.
The matrix $P$ supplies a pointwise relative quadratic-form comparison;
neither this geometry nor the ridge objective is asserted invariant
under arbitrary reparameterization.
The same ambient-to-intrinsic curvature correction appears in
Absil et al.~\cite[Theorem 1, (10)]{absil} and Masiha et
al.~\cite[(8)]{masiha}.

We call any point with $Q\le C\lambda\varepsilon^2$ a
\emph{low-cost point}. It need not be stationary or optimal.

For later statements, define the admissible cost cap
\[
 C_*(s,M)=\min\{s^4/(8M^2),\ s^2/(128M^2)\},
 \qquad s,M>0.
\]
The constants $s$ and $M$ respectively control the smallest Jacobian
singular value and the second derivative.

\section{Uniform coexistence under center and label changes}\label{sec:robust}
We fix one ridge interval before varying the center and labels
independently. Every resulting objective has accurate global minimizers
and an indefinite-Hessian point in the same low-cost set.
The theorem constructs an exact-fit section as the labels vary.
Persistence of curvature magnitude along this section is an assumption;
the relevant tangent directions may rotate.

\begin{theorem}[Uniform coexistence]\label{thm:robust}
Let $p>n\ge1$, $0<\varepsilon\le1$, and $s,M,r>0$, $L\ge0$.
Let $G:\R^p\to\R^n$ be continuous and $C^{2,1}$ on a neighborhood
of $\overline B(x_0,R)$, where $R\ge r\varepsilon$,
$\norm{D^2G}\le M$, and $D^2G$ has Lipschitz constant at most $L$
on that ball. Suppose $J_0=DG(x_0)$ is onto,
$\sigma=\sigma_n(J_0)\ge s\varepsilon$, and
$B_0=\norm{\II_{x_0}}>0$.
Set $s_*=s/2$ and $0<C\le C_*(s_*,M)$. Define
\[
\begin{split}
\delta&=\min\{\sqrt C/8,\ s/(16M)\},\qquad
\alpha=\delta/2,\\
\beta&=\min\{s\delta/4,\ s^2/(32M),\ rs/8,\ 1/4\},
\end{split}
\]
and assume $R\ge(\alpha+\sqrt{2C})\varepsilon$.
For each $\norm v\le\beta\varepsilon^2$, let
$y(v)=x_0+h(v)$ be the normal-section exact fit with
\[
\begin{gathered}
 h(v)\in(\ker J_0)^\perp,\qquad
 G(y(v))=G(x_0)+v,\\
 \norm{h(v)}\le2\norm v/\sigma.
\end{gathered}
\]
The stated bounds guarantee existence and uniqueness among
$h\in(\ker J_0)^\perp$ with $\norm h\le2\norm v/\sigma$.
Assume that its curvature magnitude persists:
\begin{equation}\label{eq:robust-persistence}
 B_0/2\le B_{y(v)}\le2B_0
 \qquad(\norm v\le\beta\varepsilon^2).
\end{equation}
Put $m_0=\sigma/2$, $b_0=B_0/2$, $D_0=Lm_0+8M^2$, and
\begin{equation}\label{eq:robust-window}
\begin{split}
\Lambda=\min\bigl\{&m_0^3b_0/(128M),
 m_0^4b_0^2/(64D_0),\\
 &m_0^2/64,\ Cm_0^2b_0^2\varepsilon^2/16,\\
 &\sqrt C\,\varepsilon m_0^2b_0/16,
 Rm_0^2b_0/16\bigr\}>0,\\
W&=\min\{\Lambda,\varepsilon^2\}.
\end{split}
\end{equation}
For every independent pair $d=x_0+u$, $\norm u\le\alpha\varepsilon$,
and $Y=G(x_0)+v$, $\norm v\le\beta\varepsilon^2$, and every
$0<\lambda\le W$, use the same $Q$ in~\eqref{eq:objective} and set
\[
\begin{aligned}
\mathcal L_\lambda&=\{x:Q(x)\le C\lambda\varepsilon^2\},\\
A(x)&=B_x\norm{\nabla Q(x)}/\lambda,\\
P_x&=DG(x)^TDG(x)+\lambda I.
\end{aligned}
\]
Then the following hold for this one objective and low-cost set.
\begin{enumerate}
\item Every $x\in\mathcal L_\lambda$ lies in
$\overline B(x_0,R)$, has
$\sigma_n(DG(x))\ge29s\varepsilon/32>s_*\varepsilon$, and
satisfies the local smoothness premises of Theorem~\ref{thm:point}.
\item A global minimizer exists. Every global minimizer
$z_\lambda$ has
$Q(z_\lambda)\le C\lambda\varepsilon^2/128$ and
$A(z_\lambda)=0$, so the test passes even when $Q(d)>0$.
\item At every $x\in\mathcal L_\lambda$ with $A(x)\le1/8$,
\[
\begin{gathered}
 H_x\succeq\lambda I/2,\qquad c_*=(1+\sqrt2)/8,\\
 (1-c_*)P_x\prec H_x\prec(1+c_*)P_x,\qquad
 \rho(x)<c_*.
\end{gathered}
\]
\item The same set contains $x_\lambda\in B(x_0,R/2)$
with strict low cost and indefinite complete Hessian. At this point
\[
\begin{gathered}
 \frac{21}{16}\le A(x_\lambda)\le\frac{45}{16},\\
 \frac{7\lambda}{8B_0}\le\norm{\nabla Q(x_\lambda)}
 \le\frac{9\lambda}{2B_0},\\
 \frac{3B_0}{8}\le B_{x_\lambda}\le\frac{5B_0}{2}.
\end{gathered}
\]
Moreover, $\rho(x_\lambda)\ge15/8$.
The witness $x_\lambda$ is nonstationary.
\end{enumerate}
\end{theorem}

For the last bound, the proof gives a unit
$t\in\ker DG(x_\lambda)$ with
$t^TH_{x_\lambda}t\le-7\lambda/8$, while
$t^TP_{x_\lambda}t=\lambda$. Its relative Rayleigh quotient
in~\eqref{eq:rho-intro} is at least $15/8$.

Under Theorem~\ref{thm:robust}, one $W$ works for every allowed independent
center and label pair and every $0<\lambda\le W$.
For each such objective, global minimizers exist. Any global minimizer
$z_\lambda$ and some indefinite point $x_\lambda$ lie in the same
low-cost set and satisfy
\[
 \boxed{\rho(z_\lambda)<c_*,\qquad \rho(x_\lambda)\ge15/8}.
\]
The points may vary with these choices. Values between the two regimes
remain unclassified.

For each allowed $d,Y,\lambda$, the exact-fit anchor $y(v)$ has low cost
for the same $Q$. Thus every global minimizer lies in
$\mathcal L_\lambda$. The ridge term confines that entire set near
$x_0$, where the derivative bound preserves full row rank. The tangent
identity~\eqref{eq:tangent} bounds the certificate's absolute tangent
correction and supplies the witness's signed curvature pairing.

At the anchor, the actual residual is zero. The witness proof in
Appendix~\ref{app:robust} chooses a unit tangent and a separate residual.
Their pairing against the anchor's second derivative is $-2\lambda$.
That residual is realized nearby for the same $d,Y,\lambda$.
The common ridge cap controls
residual cost, displacement, and changes in tangent direction and
curvature. The bounds on $A(x_\lambda)$ use gradient and curvature
estimates tied to the same moving anchor. Multiplying the separate
$B_0$-relative intervals
weakens them. These bounds give no exact threshold for $A$.

Appendix~\ref{app:nearflat-example} gives a near-flat example with fixed radii.

The persistence premise compares curvature magnitudes along the
exact-fit section. The construction selects a new tangent and residual
direction at each label; it does not need a fixed orientation.
This comparison permits center and label radii independent of $B_0$,
while the ridge cap in~\eqref{eq:robust-window} still depends on $B_0$.
Proposition~\ref{prop:rotating} illustrates the changing direction.
Proposition~\ref{prop:bump} gives a stronger boundary for the stated
smoothness bounds. For sufficiently small $\varepsilon$ and
$0<\lambda\le\varepsilon^4/(8C)$, after an arbitrarily small
normalized label shift, its entire nonempty low-cost set lies where
$G$ is locally affine.
There $E=0$, so $H=P\succ0$ and $\rho=0$ at every point of that set.
This rules out a uniform label-radius conclusion from smoothness
alone; it does not prove that the exact persistence condition above
is necessary.

\section{A pointwise curvature certificate}\label{sec:upper}
The certificate depends on where the residual correction acts.
On $T$, $J^TJ$ vanishes, so the gate bounds $\|E_{TT}\|$ by $\lambda/4$.
On $N$, $J^TJ$ has a gap of at least $s^2\varepsilon^2$.
This gap can absorb a correction larger than $\lambda$.
The mixed block reduces the tangent margin through the Schur complement.
The two caps on $C$ protect the normal gap and keep this loss at most $\lambda/32$.
The tangent identity~\eqref{eq:tangent} explains the $\lambda/B$
gradient tolerance for $B>0$.
The test uses second-order information about $G$.

\begin{theorem}[Pointwise certificate]\label{thm:point}
Let $G$ be $C^2$ on a neighborhood of the point in question. Fix $s,M>0$ and
\[
 0<C\le C_*(s,M):=\min\{s^4/(8M^2),s^2/(128M^2)\}.
\]
Let $0<\varepsilon\le1$, $0<\lambda\le\varepsilon^2$. At any point
with $Q\le C\lambda\varepsilon^2$, $\sigma_n(J)\ge s\varepsilon$,
and $\norm{D^2G}\le M$, the condition
\begin{equation}\label{eq:gate}
 \min\{K\sqrt{2C\lambda}\varepsilon,
           B(\norm b+\lambda\sqrt{2C}\varepsilon)\}\le\lambda/4
\end{equation}
implies $H\succeq\lambda I/2$.
Sufficient alternatives are $B=0$;
$\norm b\le\lambda/(8B)$ if $B>0$;
$\norm b\le s\lambda\varepsilon/(8\bar K)$ if
$K\le\bar K\le M$, $\bar K>0$; or the cost-only bound
\begin{equation}\label{eq:nearflat}
 K\le\frac{\sqrt\lambda}{4\sqrt{2C}\varepsilon}.
\end{equation}
A positive envelope $B\le\bar B\le M/(s\varepsilon)$ may replace $B$
in the gradient test. A zero envelope needs no gradient test.
\end{theorem}

\begin{proof}
Cost gives $\norm r\le\sqrt{2C\lambda}\varepsilon$ and
$\norm{\theta-d}\le\sqrt{2C}\varepsilon$. Put
$\delta_E=\norm E\le M\sqrt{2C\lambda}\varepsilon
\le s^2\varepsilon^2/2$.
On tangent pairs the level-set identity gives
\begin{equation}\label{eq:tangent}
 E[v,w]=-\langle J^Tr,\II(v,w)\rangle
       =-\langle b-\lambda(\theta-d),\II(v,w)\rangle.
\end{equation}
Thus (\ref{eq:gate}) bounds $\norm{E_{TT}}$ by $\lambda/4$.
In orthonormal $N\oplus T$ coordinates, the normal block of
$H-\lambda I/2$ is at least
$s^2\varepsilon^2I/2+\lambda I/2$; its tangent block is at least
$\lambda I/4$; and its mixed block has norm at most $\delta_E$.
The Schur loss is at most
\[
 \frac{2\delta_E^2}{s^2\varepsilon^2}
 \le\frac{4CM^2}{s^2}\lambda\le\lambda/32.
\]
Hence $H\succeq\lambda I/2$, including when $T$ is empty.
Finally $B\le K/(s\varepsilon)\le M/(s\varepsilon)$ and
$M\sqrt{2C}/s\le1/8$. The ridge-displacement term in the second
branch is at most $\lambda/8$, proving the alternatives. The first
branch gives (\ref{eq:nearflat}).
\end{proof}

The same conditions also compare the full Hessian with the ridge
Gauss--Newton matrix.

\begin{corollary}[Pointwise relative certificate]\label{cor:relative-gate}
Assume all the hypotheses of Theorem~\ref{thm:point}, including the
gate~\eqref{eq:gate}, at the point under consideration. Set
$P=J^TJ+\lambda I$, $H=\nabla^2Q$, and
$c_*=(1+\sqrt2)/8$. Then
\begin{equation}\label{eq:cstar}
 \boxed{(1-c_*)P\prec H\prec(1+c_*)P}.
\end{equation}
Both inequalities are strict for every positive $\lambda$. In every
dimension,
\[
 \kappa_2(P^{-1/2}HP^{-1/2})
 <\frac{1+c_*}{1-c_*}<1.865.
\]
If $p=n$, the stronger bound
$\norm{P^{-1/2}(H-P)P^{-1/2}}_2\le1/16$ holds.
The constant $c_*$ is an optimal supremum over the stated pointwise
least-squares class when the map and ridge vary. This does not assert
sharpness for one fixed map, a fixed-reference weak-jet family, or a
whole tube.
\end{corollary}

The gate bounds the absolute size of the tangent correction, not only
its negative part. This is what gives both sides of~\eqref{eq:cstar}.
A one-sided tangent lower bound alone does not give the same upper
bound. The general envelope and the proof are in
Appendix~\ref{app:relative-topic}.
A separate bound on each normal, mixed, and tangent block already
gives $\rho\le5/16$. The smaller $c_*$ uses their shared bound
$\|E\|$.

Equivalently, \eqref{eq:rho-intro} gives $\rho<c_*$ at the evaluated point.
At a qualifying stationary point, the ordinary full-correction
bound already gives $\rho\le1/8$: indeed
$\|J^Tr\|=\lambda\|\theta-d\|$,
$\|r\|\le\lambda\|\theta-d\|/(s\varepsilon)$, and the cost
and cap imply $\|E\|/\lambda\le M\sqrt{2C}/s\le1/8$.
The directional distinction is therefore clearest at nonstationary
points.

\begin{corollary}[Quality of the local quadratic model]
\label{cor:quadratic-decrease}
At a point with $b\ne0$, let
$m_H(h)=b^Th+h^THh/2$, $h_P=-P^{-1}b$, and
$h_H=-H^{-1}b$. If $\rho<1$, then
\begin{equation}\label{eq:quadratic-decrease}
 -m_H(h_P)\ge(1-\rho^2)[-m_H(h_H)]>0.
\end{equation}
\end{corollary}
The proof is in Appendix~\ref{app:quadratic-model-proof}.

Both steps are evaluated in the same quadratic model $m_H$.
The bound compares their predicted decreases. It does not
claim that $Q(\theta+h_P)$ decreases or that a Gauss--Newton
iteration succeeds.

\section{An output-alignment example}\label{sec:multioutput}

The next fixed map has two outputs. Aggregate derivative norms lose
which output carries the residual. This makes several valid upper
bounds inconclusive, while the directional gate passes.

\begin{proposition}[A fixed two-output map]\label{prop:multioutput}
All unmarked norms are Euclidean vector and spectral matrix norms;
$\|D^2G\|$ is the bilinear operator norm. Fix
\[
 G(u,v,z)=(u+z^2/2,v^2/2).
\]
For every $0<\varepsilon\le1/128$, $1\le t\le2$, and
$1/32\le a\le1/16$, take
$x=d=(0,\varepsilon,0)$, $\lambda=t\varepsilon^4$,
and $Y=(0,\varepsilon^2/2-a\varepsilon^3)$.
Use $s=M=1$ and $C=1/512$ in Theorem~\ref{thm:point}.
Then $n=2<p=3$, $\sigma_2(J)=\varepsilon$, the cost and ridge
caps hold, and $b\ne0$. At $x$, $K=B=1$ and
\[
 B(\|b\|+\lambda\sqrt{2C}\,\varepsilon)
 \le129\lambda/2048<\lambda/8.
\]
The cost branch of~\eqref{eq:gate} is above $\lambda/4$,
and $K\|r\|\ge2\lambda>\lambda/4$.

To compare whole-correction bounds, freeze $P$ at $x$. Set
$\Delta=x-d$, $a_0=J^Tr=b-\lambda\Delta$,
$j=\sigma_2(J)$, and
\[
\begin{split}
M_P&=\sup_{\|h\|_P=\|k\|_P=1}\|D^2G[h,k]\|,\\
j_P^2&=\lambda_{\min}(JP^{-1}J^T),\\
\|h\|_P^2&=h^TPh,\quad\|v\|_{P^{-1}}^2=v^TP^{-1}v.
\end{split}
\]
The four valid upper envelopes
\[
 \frac{M\|r\|}{\lambda},\quad
 \frac{M\|a_0\|}{j\lambda},\quad
 M_P\|r\|,\quad
 \frac{M_P\|a_0\|_{P^{-1}}}{j_P}
\]
all equal $a/(t\varepsilon)\ge2$. They therefore do not
certify $\rho<1$. Here $\Delta=0$, so replacing $a_0$ by the
usual triangle bound adds no loss. Yet $E_{TT}=0$ and
\begin{equation}\label{eq:multioutput-rho}
 \rho=\frac{a\varepsilon}{1+t\varepsilon^2}<1/2048.
\end{equation}
The componentwise residual bound
\begin{equation}\label{eq:componentwise}
 \sum_{i=1}^2|r_i|
 \|P^{-1/2}D^2G_iP^{-1/2}\|_2=\rho
\end{equation}
also passes. Moreover $E\succeq0$, so positivity is immediate
from its sign. This example separates the named aggregate bounds,
not all residual-aware bounds or methods.
\end{proposition}

At $x$, the first output supplies the fiber's bending.
Only the second output carries residual, so $E_{TT}=0$.

The ordinary residual and raw-gradient envelopes are the squared-loss
specialization of Semenov, Jaggi, and Doikov~\cite[Example 7,
(13)--(14); Appendix H.2, Example 12, (80)--(81)]{sjd}. The two $P$-weighted
envelopes follow from the same argument after a fixed whitening at
the evaluated point. The componentwise estimate uses the residual
of each output separately. This pointwise comparison shows how output
alignment affects the certificate.
The proof is in Appendix~\ref{app:multioutput}.

\section{Relation to a distance-tube test}\label{sec:tube}
Masiha et al.~\cite[Lemma E.1]{masiha} prove an ambient
full-Hessian bound $\nabla^2g(x)+\gamma I\succeq\gamma I/2$
inside a distance tube around a compact minimizer manifold $S$.
Here $g=\norm{G-Y}^2/2$. Their Proposition E.6 bounds the
second fundamental form \emph{on} $S$. Our $B(x)$ describes the
current output fiber at $x$. These are different curvature objects.

There is also a useful gradient route into their tube. Suppose a point
$x$ is already localized near $S$, the normal Hessian of $g$ along
each nearest-point segment is at least $c/2$, and $y$ is a nearest
point in $S$. Integrating along the normal segment, as in the local
argument of Masiha et al.~\cite[Lemma E.7, (17)]{masiha}, gives
\begin{equation}\label{eq:distancegradient}
 \operatorname{dist}(x,S)\le \frac2c\norm{\nabla g(x)}
 \le\frac2c\bigl(\norm{\nabla Q(x)}+\lambda\norm{x-d}\bigr).
\end{equation}
Localization is a premise of this inequality. When the right side
lies in the radius of Lemma E.1, that lemma yields the same full-Hessian
margin. The proof of E.1 also permits the refined radius
$\min\{\rho,\lambda/(2L_{g,3})\}$ when $L_{g,3}>0$; here
$\rho$ is that source's tube radius.
Thus this prior route can give a $\lambda/B$ gradient scale when its
normal-gap, third-derivative, localization, and curvature-comparison
conditions allow it.

\begin{proposition}[A negative residual outside two sufficient domains]
\label{prop:negative-separator}
Let $0<a\le1/64$, $0<\varepsilon\le1/32$, and set
\begin{align*}
 G_a(x,y)&=x^2+ay^2,& d&=(\varepsilon,0),\\
 Y&=\varepsilon^2,& \lambda&=a^2\varepsilon^4.
\end{align*}
Define $X=\sqrt{\varepsilon^2-\lambda/(32a)}$ and $z=(X,0)$.
This point satisfies Theorem~\ref{thm:point} with
$C=1/512$, $s=1$, and $M=2$. In particular,
\begin{gather*}
 Q(z)<C\lambda\varepsilon^2,\qquad H(z)\succeq\lambda I/2,\\
 \norm{\nabla Q(z)}<\frac{\lambda}{8B(z)}.
\end{gather*}
Its residual correction is negative definite, with
\begin{equation}\label{eq:negative-separator-correction}
 \begin{gathered}
 E(z)=-\frac{\lambda}{16a}\operatorname{diag}(1,a),\\
 \frac{\norm{E(z)}}{\lambda}=\frac1{16a}.
 \end{gathered}
\end{equation}
Let $S=\{(x,y):x^2+ay^2=\varepsilon^2\}$ and
$g=\norm{G_a-Y}^2/2$. For any admissible positive tube radius $\rho$ and
third-derivative bound $L_{g,3}>0$ in
Masiha et al.~\cite[Lemma E.1]{masiha}, write
\begin{align*}
 r_{\rm literal}&=\min\{\rho,\lambda/(2\max\{L_{g,3},1\})\},\\
 r_{\rm refined}&=\min\{\rho,\lambda/(2L_{g,3})\}.
\end{align*}
Then
\begin{equation}\label{eq:negative-separator-distance}
 \begin{gathered}
 \operatorname{dist}(z,S)=\frac{\lambda}{32a(\varepsilon+X)}
 >r_{\rm literal},\\
 \frac{\operatorname{dist}(z,S)}{r_{\rm refined}}>\frac3{8a}.
 \end{gathered}
\end{equation}
\end{proposition}

The proof is in Appendix~\ref{app:negative-separator}.
Thus the point passes our certificate while failing the absolute
residual-norm margin $\norm E<\lambda$
\cite[Proposition 5]{kharel} and these two isotropic tube tests.
The refined radius follows from the proof of Lemma E.1, not its
displayed statement. The comparison uses Assumptions 1 and 3 of
Masiha et al. for each fixed positive $a,\varepsilon$; their global
smoothness Assumption 4 fails here. The two ratios diverge as
$a\downarrow0$ at fixed $\varepsilon$, but the source constants are
not uniform in that limit.

This example illustrates control of a large negative normal correction;
its mixed block is zero. It does not show that gradient information is
necessary: a directional test using the tighter observed cost also
certifies this point. The separation concerns the named sufficient
conditions, not every cost-based test, every prior method, or the full
region of positive curvature.

Appendix~\ref{app:certificate-comparison} compares the specified sufficient tests.

\section{A structural singular-interpolation family}\label{sec:bridge}

The pointwise certificate assumes a Jacobian lower bound. Under the
derivative conditions below, the Jacobian has full row rank at every
low-cost point. These weak-jet conditions concern a rank-deficient
reference. For each allowed map, scale, center, label, and ridge
weight, the construction also gives an exact interpolator whose
objective value is strictly below the low-cost cap. Coexistence still
needs the separate common $C^{2,1}$ bounds and curvature-persistence
premise of Theorem~\ref{thm:robust}.

Fix a globally $C^2$ reference $F$, a point $c$, and $A=DF(c)$ of rank
$g<n$. Choose $g$ independent physical output rows $R$, let $D$ be the
remaining rows, and define the fixed row matrix $L$ by $A_D=LA_R$.
Let $Z$ have orthonormal columns spanning $\ker A$. Empty retained
blocks are omitted when $g=0$. Define the tensor field and quadratic map
\begin{align*}
 \mathcal K_F&=D^2F_D-LD^2F_R,\\
 q(z)&=-\tfrac12\mathcal K_F(c)[Zz,Zz],\\
 \chi(\ell)&=\ell_D-L\ell_R.
\end{align*}
Fix $z_0,\ell_0$ such that $q(z_0)=\chi(\ell_0)$ and
$\rank Dq(z_0)=n-g$, and put $u_0=Zz_0$.
Fix $r_\star>0$ and $M\ge1+\sup_{\overline B(c,r_\star)}\norm{D^2F}$.

\begin{lemma}[Cost-only rank and flexible cap]\label{lem:bridge}
There are fixed $s>0$, finite weak/strong Jacobian bounds, and $\bar C>0$
depending only on these data with the following property. For every
$0<C\le\bar C$, one can choose $r_0\le r_\star$, $\kappa>0$,
$R_c,R_\ell>0$ and $0<\varepsilon_0\le1$, with
the tube-radius parameter $\rho=R_c+\sqrt{2C}<1/2$, so that all
assertions below are uniform.
For every $0<\varepsilon<\varepsilon_0$, let $G$ be globally $C^2$ with
\begin{align}
 &\norm{DG(c)-A}\le\kappa,\quad
 \sigma_{g+1}(DG(c))\le\kappa\varepsilon,\nonumber\\
 &\sup_{\overline B(c,r_0)}\norm{D^2G}\le M,\label{eq:base}\\
 &\sup_{\overline B(c,r_0)}
 \norm{(\mathcal K_G-\mathcal K_F)[Z\cdot,Z\cdot]}\le\kappa.
 \label{eq:weakjet}
\end{align}
For all $u\in\overline B(u_0,R_c)$, $\ell\in\overline B(\ell_0,R_\ell)$,
$0<\lambda\le\varepsilon^2$, set $d=c+\varepsilon u$ and
$Y=G(c)-\varepsilon^2\ell$. Every whole-space point with
$Q\le C\lambda\varepsilon^2$ lies in
$\mathcal T_\varepsilon=\overline B(c+\varepsilon u_0,\rho\varepsilon)
\subset B(c,r_0)$ and satisfies
\begin{equation}\label{eq:Jbands}
 \begin{cases}
 a_s\le\sigma_j(J)\le b_s,&1\le j\le g,\\
 s\varepsilon\le\sigma_j(J)\le b_w\varepsilon,&g<j\le n,
 \end{cases}
\end{equation}
where $a_s,b_s$ are omitted if $g=0$.
There is an exact interpolator with $Q<C\lambda\varepsilon^2$ for every
tuple. The constants in~\eqref{eq:Jbands} are fixed before the final cap
$C$ is chosen. Only the reference Hessian needs a continuity modulus;
maps at different scales may be unrelated.
\end{lemma}

At $c$, the reduced reference output $F_D-LF_R$ has zero first
derivative. Its quadratic term along $\ker A$ is $-q$. Thus a kernel
move of size $O(\varepsilon)$ changes that reference output by
$O(\varepsilon^2)$, possibly by zero. For an allowed $G$, the weak first
derivative need not vanish. The Jacobian bounds control this leakage
in moving weak coordinates. The weak-jet bound controls the quadratic
comparison.

Under all the lemma's hypotheses, the proof in
Appendix~\ref{app:bridge} solves the retained equations, when present,
before the reduced equations. The reduced
solve uses $q(z_0)=\chi(\ell_0)$ and surjective $Dq(z_0)$ to produce an
exact fit. Separately, cost places the weak coordinate of every low-cost
point near $z_0$, where $Dq$ remains full row rank.
After division by $\varepsilon$, the weak Jacobian Schur block stays
close to $-Dq$ there. Thus all $n-g$ weak singular values have order
$\varepsilon$ throughout the low-cost set. Right
singular vectors for these nonzero singular values are normal to the
current fiber. Tangent directions have zero Jacobian response. The
interpolation tolerances are selected after the cap.

\begin{corollary}[Indexed spectrum in the structural family]\label{cor:family}
Under all hypotheses of Theorem~\ref{thm:point}, including its cap
$C\le C_*(s,M)$, local $C^2$ and second-derivative bound, ridge range,
rank lower bound, low-cost bound, and gate~\eqref{eq:gate}, suppose in
addition that the two-sided Jacobian bands~\eqref{eq:Jbands} hold with
fixed constants. Then there are fixed positive $c_H,C_H$ such that,
for sufficiently small $\varepsilon$, the increasingly ordered
eigenvalues obey
\begin{equation}\label{eq:Hbands}
 \begin{array}{ll}
 c_H\lambda\le\lambda_j(H)\le C_H\lambda,&1\le j\le p-n,\\
 c_H\varepsilon^2\le\lambda_j(H)\le C_H\varepsilon^2,&p-n<j\le p-g,\\
 c_H\le\lambda_j(H)\le C_H,&p-g<j\le p.
 \end{array}
\end{equation}
Choose the cap of Lemma~\ref{lem:bridge} below $C_*(s,M)$. Then these
conclusions hold at every low-cost point of its structural family satisfying
the adaptive test. In particular $\norm b\le\gamma\lambda\varepsilon$
with $\gamma=s/(8M)$ suffices. Every tuple has a certified stationary
global minimizer strictly below the cap.
\end{corollary}

The proof is in Appendix~\ref{app:indexed-spectrum-proof}.

The indexed estimates do not assert asymptotic
separation of the ridge and weak bands when
$\lambda\asymp\varepsilon^2$. Without the two-sided
Jacobian bands, Theorem~\ref{thm:point} still gives a tangent
Rayleigh-quotient upper bound $5\lambda/4$ and a lower bound
$s^2\varepsilon^2/2$ on the normal compression. Min--max and
interlacing turn these into indexed eigenvalue bounds. They do not
identify invariant tangent eigenvectors. The complete classification
needs~\eqref{eq:Jbands}.

Appendix~\ref{sec:worked} gives the quadratic calculation.
\section{Related work and scope}\label{sec:related}
The residual-inclusive Hessian \cite[Theorem 4.1]{eriksson},
level-set curvature \cite[Lemma 2.2.1]{feehan}, and embedded-manifold
correction \cite[Theorem 1]{absil} are classical foundations.
Our conditional common interval certifies every global minimizer
while retaining a low-cost indefinite point for each allowed objective.

Masiha et al.~\cite[(8), Lemma E.1, Proposition E.6]{masiha} give
related curvature and distance-tube results, including a full ambient
positivity margin.
Semenov, Jaggi, and Doikov~\cite[Example 7, (13)--(14);
Appendix H.2, Example 12, (80)--(81)]{sjd} give the ordinary residual
and gradient bound for a nonlinear least-squares correction. Applied to
$g=\|G-Y\|^2/2$, it gives
$\|E\|\le M\|r\|\le(M/j)\|b-\lambda(\theta-d)\|$.
Fixing $P$ and changing coordinates gives a pointwise $P$-metric
bound. Our gate uses additional information from the current
prediction level set. The frozen-metric bound also certifies the
scalar point in Proposition~\ref{prop:negative-separator}, where
$\rho=1/16$.

Liu, Zhu, and Belkin~\cite[Section 3, (9)--(10), Proposition 2]{liuzhubelkin}
and Rebjock and Boumal~\cite[Section 4.2]{rebjock} give qualitative
nearby nonconvexity results for curved sets of interpolants or minima.
For squared distance to a fixed manifold, Leobacher and
Steinicke's projection formula~\cite[Theorem C]{leobacher} also yields
a $\lambda/\text{curvature}$ scale by the calculation in
Appendix~\ref{app:prior-distance}. That inference does not cover an
arbitrary raw least-squares map. Bock et al.~\cite[Section 2,
Example 2]{bock} reverse a residual Hessian term by mirroring
measurement labels. Their construction uses a nonzero-residual
stationary point and a full-column-rank Jacobian.
Kharel et al.~\cite[Proposition 5]{kharel} give a margin when their
residual Hessian norm lies below the ridge.

Our common interval covers every allowed independent center and label pair
for all $0<\lambda\le W$. The same neighborhoods work as the ridge
tends to zero. The structural bridge builds on regular-zero theory
\cite{sussmann,arutyunov2011} and stability of inversion under small
perturbations~\cite{avakov2021,arutyunov2025}.
These perturbation results apply once the full normalized map is close
to a regular reference. Our weak-jet estimates establish this closeness
and give uniform rank throughout the low-cost set.
Appendix~\ref{app:bridge} details the connection.

The pointwise result needs local $C^2$ bounds; coexistence separately
requires common $C^{2,1}$ bounds and curvature-magnitude persistence
(Sections~\ref{app:robust} and~\ref{app:geometry-boundaries}).
Proposition~\ref{prop:rotating} permits changing orientation but shows
failure of a fixed base residual; Proposition~\ref{prop:bump} shows
smoothness alone cannot fix normalized label radii.

The results concern curvature in the fixed Euclidean parameter metric.
They do not imply convergence of a Gauss--Newton iteration.

\section{Conclusion}
Low cost and accurate curvature at every global minimizer can coexist
with an indefinite Hessian elsewhere in the same low-cost set.
Under the stated regularity and curvature-persistence assumptions,
this contrast survives independent center and label perturbations
throughout one common positive ridge interval. The perturbation
neighborhoods remain fixed as the ridge weight tends to zero.

The pointwise certificate measures the residual correction through
current level-set geometry. Its relative bound is sharp over the
stated pointwise class. The analytic examples show why aggregate
bounds can lose output alignment and why local smoothness alone
does not give the same uniform label neighborhood. Separately,
the structural result supplies rank and interpolation guarantees
near a singular reference. These results distinguish curvature
accuracy at a point from accuracy throughout a low-cost region.

\section*{Acknowledgments}
Generative AI assisted proof development, checks of intermediate
arguments, literature comparisons, and manuscript editing.
The authors verified the claims and proofs and take responsibility
for the contents.

\clearpage
\appendix
\section{Relative accuracy and a pointwise example}\label{app:relative-topic}

\subsection{Relative spectral control}\label{sec:relative}

The adaptive gate also controls the Hessian relative to the ridge
Gauss--Newton matrix. This comparison is pointwise. It does not require
the two-sided Jacobian bands of Lemma~\ref{lem:bridge}.

Let $p\ge n\ge1$, $J\in\R^{n\times p}$ have full row rank and
$\sigma_n(J)\ge j>0$. Set $T=\ker J$, $\gamma=j^2$,
$P=J^TJ+\lambda I_p$ with $\lambda>0$, and $H=P+E$ for symmetric
$E$ with $\norm E_2\le\delta$. A \emph{one-sided} tangent envelope
$0\le k\le\delta$ means $t^TEt\ge-k\norm t^2$ for $t\in T$.
An \emph{absolute} envelope means $|t^TEt|\le k\norm t^2$.

\begin{theorem}[Relative envelope]\label{thm:relative}
If $p>n$, define
\begin{equation}\label{eq:rhorel}
 \rho_-=
 \frac{\gamma k+\sqrt{\gamma^2k^2+
 4\lambda(\lambda+\gamma)\delta^2}}
 {2\lambda(\lambda+\gamma)}.
\end{equation}
This is the nonnegative root of
$\lambda(\lambda+\gamma)u^2-\gamma ku-\delta^2=0$.
A one-sided tangent envelope gives
\begin{equation}\label{eq:relative-one}
 (1-\rho_-)P\preceq H\preceq(1+\delta/\lambda)P.
\end{equation}
Both bounds are sharp in the displayed information class. If the
envelope is absolute, the sharp symmetric bound is
\begin{equation}\label{eq:relative-absolute}
 (1-\rho_-)P\preceq H\preceq(1+\rho_-)P.
\end{equation}
Each endpoint is attained in dimension $p=2,n=1$. They are attained
together in one $p=4,n=2$ instance when $\rho_-<1$. In that class,
the upper bound $(1+\rho_-)/(1-\rho_-)$ on
$\kappa_2(P^{-1/2}HP^{-1/2})$ is sharp.
These endpoint statements do not assert joint attainment in every
fixed dimension. If $p=n$, $T=\{0\}$ and the sharp norm replacement is
\begin{equation}\label{eq:relative-square}
 \norm{P^{-1/2}EP^{-1/2}}_2\le\frac{\delta}{\lambda+\gamma}.
\end{equation}
The scalar case can attain this norm endpoint even when $H$ is
singular. If $p=n=1$ and $H\succ0$, then
$\kappa_2(P^{-1/2}HP^{-1/2})=1$.
\end{theorem}

The proof is in Appendix~\ref{app:relative}. Its two-sided conclusion
uses the absolute tangent bound. From a one-sided bound alone, the
upper endpoint remains $\delta/\lambda$. The formula is a rescaling
corollary of a sharp additive envelope derived from block and Schur
arguments; the elementary two-dimensional extremizer also appears in
that additive result.

To prove Corollary~\ref{cor:relative-gate}, set
$\mu=\norm r$, $\delta=M\mu$, $j=s\varepsilon$, and
$\gamma=s^2\varepsilon^2$. If $p>n$, the pointwise hypotheses give
the absolute tangent envelope
\begin{equation}\label{eq:kgate}
\begin{split}
 k=\min\{&\delta,\ K\sqrt{2C\lambda}\,\varepsilon,\\
         &B(\norm b+\lambda\sqrt{2C}\,\varepsilon)\}.
\end{split}
\end{equation}
It satisfies $k\le\lambda/4$ and
$\delta^2\le\gamma\lambda/64$. Substituting these bounds in
\eqref{eq:rhorel} gives $\rho_-<c_*$ for every positive ridge.
For $p=n$, \eqref{eq:relative-square} is at most $1/16$.
Appendix~\ref{app:relative} proves these estimates and the varying-map
supremum claim.

\subsubsection{Proof and sharpness}\label{app:relative}

Write $N=T^\perp=\operatorname{ran}J^T$ and
$\gamma=j^2$. An additive block bound gives, for
$p>n$, a symmetric $E$ with $\|E\|_2\le\delta$ and
$E|_T\succeq-kI_T$ obeys
\[
 J^TJ+\lambda I+E\succeq(\lambda-\chi)I,
 \qquad\chi(\gamma,k,\delta)
 =\frac{\sqrt{\gamma^2+4\gamma k+4\delta^2}-\gamma}{2}.
                                                               \tag{P1}
\]
This additive frontier follows from the block and Schur argument
below. Its extremizer also gives the sharp relative endpoint.
It holds for every positive ridge and every row-rank lower bound.

Put $a=\chi$. Its defining identity is
$a(a+\gamma)=\delta^2+\gamma k$. It suffices to show
$aI+E+J^TJ\succeq0$. Write $B_a=aI+E$ in $T\oplus N$ blocks
$\left[\begin{smallmatrix}R&F\\F^T&C\end{smallmatrix}\right]$.
The norm bound gives $(a-\delta)I\preceq B_a\preceq(a+\delta)I$,
and $R\succeq(a-k)I$. If $a\ge\delta$, then $B_a\succeq0$.
The remaining case has $k<a<\delta$. Indeed, $a<k$ contradicts
the defining identity and $\delta\ge k$; equality $a=k$ requires
$a=k=\delta$ and belongs to the first case.

Set $l=a-\delta<0$, $u=a+\delta>0$ and $t=a-k>0$.
The Schur complement of $R$ is $S=C-F^TR^{-1}F$.
For a unit $y\in N$, its minimizing tangent vector is
$x=-R^{-1}Fy$. If $x\ne0$, compress $B_a$ to the plane spanned
by $x/\norm x$ and $y$, obtaining
$\left[\begin{smallmatrix}r_0&f_0\\f_0&c_0\end{smallmatrix}\right]$.
Stationarity gives $f_0=-\norm x\,r_0$ and hence
$y^TSy=c_0-f_0^2/r_0=\det(B_a|_{\mathrm{plane}})/r_0$.
Both eigenvalues of this compression lie in $[l,u]$.
Their product is at least $lu$. Since $r_0\ge t$ and $lu<0$,
$y^TSy\ge lu/t$. If $x=0$, then $y^TSy=y^TCy\ge l\ge lu/t$.
Thus $S\succeq(lu/t)I_N$. Finally
$J^TJ|_N\succeq\gamma I_N$ and
$\gamma t\ge\delta^2-a^2=-lu$ by the defining identity.
The Schur complement of $R$ in $aI+E+J^TJ$ is nonnegative.
This proves (P1), including the case $a=0$ by the first case.

Fix $u>0$ and apply (P1) to $uP+E$ using
$J_u=\sqrt u\,J$, ridge $u\lambda$, and
$\gamma_u=u\gamma$. Thus $uP+E\succeq0$ whenever
$u\lambda\ge\chi(u\gamma,k,\delta)$.
The function $a\mapsto a(a+u\gamma)$ is increasing on $a\ge0$,
and the defining identity for $\chi$ is
$\chi(\chi+u\gamma)=\delta^2+u\gamma k$. Therefore
\[
 u\lambda\ge\chi(u\gamma,k,\delta)
 \quad\Longleftrightarrow\quad
 \lambda(\lambda+\gamma)u^2-\gamma ku-\delta^2\ge0.
                                                               \tag{P2}
\]
For $\delta>0$ the polynomial is negative at zero, then crosses zero
once on $u\ge0$, at $\rho_-$ in~\eqref{eq:rhorel}. Consequently
$E\succeq-\rho_-P$. If $\delta=0$, then $k=0$, $E=0$,
and the assertion holds with $\rho_-=0$.

The universal norm bound
$E\preceq\delta I\preceq(\delta/\lambda)P$ gives the upper side
of~\eqref{eq:relative-one}. If the tangent premise is absolute,
apply the lower argument to $-E$. Both signs have norm at most
$\delta$ and tangent lower envelope $k$. This gives
\eqref{eq:relative-absolute}. One-sided information alone leaves the
upper endpoint at $\delta/\lambda$.

For sharpness at fixed parameters, use $p=2,n=1$ with coordinates
$(t,y)$, $J=(0,j)$, and
\[
 P=\operatorname{diag}(\lambda,\lambda+\gamma),\quad
 q=\sqrt{\delta^2-k^2},\quad
 E_0=\begin{pmatrix}-k&q\\q&k\end{pmatrix}.
                                                               \tag{P3}
\]
The matrix $E_0$ has eigenvalues $\pm\delta$ and tangent value
$-k$. For every $u\ge0$,
\[
 \det(uP+E_0)
 =(u\lambda-k)(u(\lambda+\gamma)+k)-q^2
 =\lambda(\lambda+\gamma)u^2-\gamma ku-\delta^2.
                                                               \tag{P4}
\]
The determinant vanishes at $u=\rho_-$, so
$P^{-1/2}E_0P^{-1/2}$ has eigenvalue $-\rho_-$.
The sign reversal $-E_0$ has the same norm and absolute tangent
bound and gives eigenvalue $+\rho_-$. In the one-sided class,
$E=\operatorname{diag}(\delta,0)$ instead gives the positive
relative endpoint $\delta/\lambda$ while satisfying
$t^TEt\ge-k\|t\|^2$ for any $k\ge0$. Moreover
\[
 \rho_-\le\delta/\lambda,
\]
because the polynomial in (P2) at $u=\delta/\lambda$ equals
$\gamma\delta(\delta-k)/\lambda\ge0$; it is strictly positive
when $k<\delta$ and $\delta>0$. Thus the two-sided optimum with
one-sided tangent information is $\delta/\lambda$.

To attain both endpoints together, take two independent
$(t,y)$ blocks with $J$ singular value $j$ on each normal coordinate
and $E=\operatorname{diag}(E_0,-E_0)$. Its spectral norm remains
$\delta$, its tangent restriction has norm $k$, and the relative
correction has eigenvalues $-\rho_-$ and $+\rho_-$. If
$\rho_-<1$, $H$ is positive definite, and the extremal
preconditioned eigenvalues give condition number exactly
$(1+\rho_-)/(1-\rho_-)$. These witnesses range over an abstract
pointwise envelope class; they do not assert extremality for a fixed
architecture or whole-tube conditions.

When $p=n$, $T=\{0\}$ and
$J^TJ\succeq\gamma I$. Hence $P\succeq(\lambda+\gamma)I$,
and $-\delta I\preceq E\preceq\delta I$ gives~\eqref{eq:relative-square}.
Taking $J=jI$ and $E=\pm\delta I$ attains either norm endpoint.
The norm endpoint may have singular $H$: for $p=n=1$, take
$E=-(\lambda+j^2)$ and $\delta=\lambda+j^2$ to get $H=0$.
When $p=n=1$ and $H\succ0$, the scalar matrix
$P^{-1/2}HP^{-1/2}$ has spectral condition number one.

\paragraph{The pointwise gate and its uniform constant.}

At the point $\theta$, the least-squares Hessian identity gives
\[
 E(v,w)=\langle r,D^2G(\theta)[v,w]\rangle.
                                                               \tag{P5}
\]
The cost inequality yields
\[
 \mu=\|r\|\le\sqrt{2C\lambda}\,\varepsilon,
 \qquad\|\theta-d\|\le\sqrt{2C}\,\varepsilon.
                                                               \tag{P6}
\]
Thus $\|E\|_2\le M\mu=\delta$ and, on $T$, the absolute
quadratic-form bound $|E(t,t)|\le K\mu\|t\|^2$ holds.
For the other branch, $JJ^\dagger=I_n$ and the definition of
$\mathrm{II}$ imply
$D^2G[t,t]=-J\mathrm{II}(t,t)$ on $T$.
Since $J^Tr=b-\lambda(\theta-d)$,
\[
 |E(t,t)|
 \le B\|J^Tr\|\|t\|^2
 \le B\bigl(\|b\|+\lambda\sqrt{2C}\,\varepsilon\bigr)
       \|t\|^2.                                             \tag{P7}
\]
The global bound also bounds this tangent form by $\delta$.
Taking the minimum gives the \emph{absolute} envelope~\eqref{eq:kgate},
and~\eqref{eq:gate} gives $k\le\lambda/4$. The second ceiling on $C$ gives
\[
 \delta^2\le 2M^2C\lambda\varepsilon^2
 \le s^2\varepsilon^2\lambda/64
 =\gamma\lambda/64.                                      \tag{P8}
\]
The first ceiling on $C$ and $\lambda\le\varepsilon^2$ remain
part of Theorem~\ref{thm:point}; (P8) and the relative
envelope do not use their full strength.

Let $c_*=(1+\sqrt2)/8$, the positive root of
$c_*^2-c_*/4-1/64=0$. At $u=c_*$, the left side of the
polynomial in (P2) satisfies
\[
 \lambda(\lambda+\gamma)c_*^2-\gamma kc_*-\delta^2
 \ge\lambda^2c_*^2+
 \gamma\lambda\bigl(c_*^2-c_*/4-1/64\bigr)
 =\lambda^2c_*^2>0.                                  \tag{P9}
\]
The nonnegative root therefore obeys $\rho_-<c_*$. For $p=n$,
(P8) and the arithmetic-geometric mean inequality give
\[
 \frac{\delta}{\lambda+\gamma}
 \le\frac{\sqrt{\gamma\lambda}}{8(\lambda+\gamma)}
 \le\frac1{16}.                                         \tag{P10}
\]
Since $c_*<1$, these bounds make $H$ positive definite.
The spectral condition-number ratio is at most
$(1+\rho)/(1-\rho)$ for the respective dimension-specific
$\rho$, and is strictly below $(1+c_*)/(1-c_*)$.

The uniform constant cannot be decreased over the stated \emph{pointwise}
least-squares class when the map varies. Set $p=4,n=2$,
$s=M=\varepsilon=1$, $C=1/128$, $0<\lambda\le1/4$,
$\mu=\sqrt\lambda/8$, $\delta=\mu$, $k=\lambda/4$, and
\[
 A=\begin{pmatrix}-k&\sqrt{\delta^2-k^2}\\
                  \sqrt{\delta^2-k^2}&k\end{pmatrix},
 \qquad \mathcal E=\operatorname{diag}(A,-A).
\]
For $u=(t_1,y_1,t_2,y_2)$ define
\[
 F(u)=\left(\mu+y_1+\frac{u^T\mathcal Eu}{2\mu},y_2\right),
 \qquad \theta=d=0,\quad Y=0.
\]
At this point,
$J=\left(\begin{smallmatrix}0&1&0&0\\0&0&0&1\end{smallmatrix}\right)$
has both singular values one, $r=(\mu,0)$, and
$Q(0)=\mu^2/2=\lambda/128=C\lambda\varepsilon^2$.
The scalar second derivative of $F_1$ is $\mathcal E/\mu$,
whose bilinear norm is $\delta/\mu=1=M$;
the second derivative of $F_2$ vanishes. The restriction to
$T=\operatorname{span}\{e_{t_1},e_{t_2}\}$ has norm
$K=k/\mu$, so the first branch of~\eqref{eq:gate} is
$K\sqrt{2C\lambda}=K\mu=k=\lambda/4$.
All other displayed pointwise hypotheses hold for
$0<\lambda\le1/4$; in particular $k\le\delta$.
The Hessian correction is exactly $\mathcal E$, and the
two relative blocks are opposites. Substitution into~\eqref{eq:rhorel} gives
\[
 \rho_\lambda
 =\frac{1+\sqrt{2+\lambda}}{8(1+\lambda)}
 \xrightarrow[\lambda\downarrow0]{}\frac{1+\sqrt2}{8}.
                                                               \tag{P11}
\]
One block has minimum $-\rho_\lambda$ and the other maximum
$+\rho_\lambda$, proving both the uniform-norm supremum and
the condition-number-ratio supremum. Each $\lambda$ defines a
different $F$; nothing here embeds them into one fixed model or
one fixed-reference, weak-jet, or whole-tube family.

\subsection{Proof for the two-output example}\label{app:multioutput}

\begin{proof}
Write $R=a\varepsilon^3$. Direct differentiation at $x$ gives
\[
 r=(0,R),\quad J=\begin{pmatrix}1&0&0\\0&\varepsilon&0\end{pmatrix},
 \quad D^2G[h,k]=(h_zk_z,h_vk_v).
\]
For unit $h,k$, the squared norm of the last vector is at most
$(h_z^2+h_v^2)(k_z^2+k_v^2)\le1$. Equality is attained on
$e_z$ or $e_v$, so $M=1$ globally. The singular values of $J$
are $1,\varepsilon$. Also $\lambda=t\varepsilon^4\le
\varepsilon^2$ and $C=1/512\le C_*(1,1)=1/128$.
The cost is $Q(x)=a^2\varepsilon^6/2\le
t\varepsilon^6/512=C\lambda\varepsilon^2$ because
$256a^2\le t$. The gradient is
$b=J^Tr=(0,a\varepsilon^4,0)\ne0$.

Here $T=\operatorname{span}\{e_z\}$,
$D^2G[e_z,e_z]=(1,0)$, and
$\II(e_z,e_z)=-e_u$. Thus $K=B=1$ and $E_{TT}=0$.
Since $\sqrt{2C}=1/16$, the gradient gate divided by $\lambda$
is $a/t+\varepsilon/16\le129/2048<1/8$. The cost gate is
$\sqrt t\,\varepsilon^3/16$, whose ratio to $\lambda/4$
is $1/(4\sqrt t\,\varepsilon)>1$. The observed tangent
residual envelope obeys $K\|r\|/\lambda=a/(t\varepsilon)
\ge2$.

In coordinate order $(u,v,z)$, the matrices are
\[
 E=\operatorname{diag}(0,R,0),\quad
 P=\operatorname{diag}(1+\lambda,\varepsilon^2+\lambda,\lambda),
 \quad H=P+E.
\]
This proves \eqref{eq:multioutput-rho}, since
$R/(\varepsilon^2+\lambda)=a\varepsilon/(1+t\varepsilon^2)$.
The inequality is strict even at $a=1/16,\varepsilon=1/128$
because the denominator exceeds one.

For the frozen metric, the exact bilinear norm is
$M_P=\max\{1/\lambda,1/(\varepsilon^2+\lambda)\}=1/\lambda$.
Moreover
\[
 JP^{-1}J^T=\operatorname{diag}
 \left((1+\lambda)^{-1},(1+t\varepsilon^2)^{-1}\right),
\quad j_P=(1+t\varepsilon^2)^{-1/2},
\]
and $\|a_0\|_{P^{-1}}=R/(1+t\varepsilon^2)^{1/2}$.
Since $\Delta=0$, $a_0=b$, and direct substitution makes all
four named envelopes $a/(t\varepsilon)$. Finally, the whitened
component Hessians are
$\operatorname{diag}(0,0,1/\lambda)$ and
$\operatorname{diag}(0,1/(\varepsilon^2+\lambda),0)$.
Only the second has nonzero residual weight, proving
\eqref{eq:componentwise}. The displayed $E$ is positive
semidefinite.
\end{proof}

\subsection{Quality of the local quadratic model}\label{app:quadratic-model-proof}
We prove Corollary~\ref{cor:quadratic-decrease}. Both steps are compared
in the same quadratic model of the full Hessian.
\begin{proof}
Set $S=P^{-1/2}(H-P)P^{-1/2}$ and $q=P^{-1/2}b$.
Then $-2m_H(h_P)=q^T(I-S)q$ and
$-2m_H(h_H)=q^T(I+S)^{-1}q$. Each eigenvalue $s$ of $S$
satisfies $(1-s)(1+s)=1-s^2\ge1-\rho^2$.
The spectral theorem proves the claim.
\end{proof}

\section{Common-domain coexistence and geometric limits}\label{app:coexistence-topic}

\subsection{Perturbed common-domain theorem}\label{app:robust}

We use the constants and notation of Theorem~\ref{thm:robust}. The proof
separates the exact-fit anchor, the entire low-cost set, the passing
global minimizer, and the recentered indefinite witness.

\paragraph{A normal exact-fit anchor for every label.}
Let $N_0=(\ker J_0)^\perp$. For $h\in N_0$ with
$\|h\|\le H_v:=2\|v\|/\sigma$, write
\[
G(x_0+h)-G(x_0)=J_0h+E_0(h),\qquad
\|E_0(h)\|\le\frac M2\|h\|^2.
\]
The Hessian bound also makes $E_0$ $MH_v$-Lipschitz on this ball.
The fixed-point map
$F_v(h)=J_0^\dagger(v-E_0(h))$ has
$\|F_v(0)\|\le H_v/2$ and Lipschitz constant
\[
\frac{MH_v}{\sigma}
\le\frac{2M\beta}{s^2}\le\frac1{16}.
\]
Because $H_v\le(2\beta/s)\varepsilon\le R/4$, this is a
contraction of the closed $H_v$-ball into itself, within the
$C^{2,1}$ neighborhood. Its unique fixed point $h(v)$ satisfies
$G(x_0+h(v))=G(x_0)+v$. For $v=0$ it is zero. The other radius
bound in $\beta$ gives
\begin{equation}\label{eq:robust-anchor-distance}
\|h(v)\|\le H_v\le\alpha\varepsilon,
\quad
\|y(v)-d\|\le(\alpha+2\beta/s)\varepsilon
\le\delta\varepsilon\le\frac{\sqrt C}{8}\varepsilon.
\end{equation}
In addition,
$\|DG(y(v))-J_0\|\le MH_v\le\sigma/16$, so
$m:=\sigma_n(DG(y(v)))\ge15\sigma/16\ge m_0$.
The assumed persistence gives
$B:=B_{y(v)}\in[b_0,2B_0]$.

\paragraph{The entire low-cost set is inside the certified domain.}
For any $x\in\mathcal L_\lambda$, its ridge term gives
$\|x-d\|\le\sqrt{2C}\varepsilon$. Therefore
\[
\|x-x_0\|\le(\alpha+\sqrt{2C})\varepsilon\le R.
\]
The second term of $C_*(s_*,M)$ implies
$\sqrt{2C}\le s/(16M)$, while
$\alpha\le s/(32M)$. Hence the Jacobian Lipschitz bound on the
segment from $x_0$ to $x$ gives
\begin{equation}\label{eq:robust-rank}
\sigma_n(DG(x))
\ge\sigma-M\|x-x_0\|
\ge\left(1-\frac1{32}-\frac1{16}\right)s\varepsilon
=\frac{29}{32}s\varepsilon>s_*\varepsilon.
\end{equation}
All these points inherit $\|D^2G(x)\|\le M$. Since
$\|DG(x)^\dagger\|\le1/(s_*\varepsilon)$, the bilinear norm of
the level-set curvature obeys
\begin{equation}\label{eq:robust-Bbound}
B_x\le\frac{M}{s_*\varepsilon}.
\end{equation}
This proves the first claim for \emph{all} low-cost points, not just
for points constructed in the proof.

\paragraph{A passing point exists for the perturbed objective.}
The anchor is an exact fit for the selected label, so by
\eqref{eq:robust-anchor-distance},
\begin{equation}\label{eq:robust-anchor-cost}
Q(y(v))=\frac\lambda2\|y(v)-d\|^2
\le\frac{C}{128}\lambda\varepsilon^2
<C\lambda\varepsilon^2.
\end{equation}
The whole-space objective is continuous and
$Q(x)\ge\lambda\|x-d\|^2/2\to\infty$ as $\|x\|\to\infty$.
It therefore attains a global minimum. Every global minimizer
$z_\lambda$ satisfies
$Q(z_\lambda)\le Q(y(v))$, so is strictly low cost. By
\eqref{eq:robust-rank} it lies in the local $C^2$ neighborhood.
Unconstrained first-order optimality gives
$\nabla Q(z_\lambda)=0$, and the statistic satisfies
$A(z_\lambda)=0$. This establishes passing-point existence even
when $Q(d)>0$.

For any low-cost $x$ satisfying $A(x)\le1/8$, the second
branch of the pointwise gate is, by \eqref{eq:robust-Bbound},
\[
B_x\bigl(\|\nabla Q(x)\|+
      \lambda\sqrt{2C}\varepsilon\bigr)
\le\frac\lambda8+
       \frac{M\sqrt{2C}}{s_*}\lambda
\le\frac\lambda4.
\]
Here the last inequality is exactly the second restriction in
$C_*(s_*,M)$. The other hypotheses of Theorem~\ref{thm:point}
hold by \eqref{eq:robust-rank}, the local Hessian bound, and
$\lambda\le W\le\varepsilon^2$. Its full-Hessian conclusion gives
$H_x\succeq\lambda I/2$. Its relative corollary gives
$(1-c_*)P_x\prec H_x\prec(1+c_*)P_x$ with
$c_*=(1+\sqrt2)/8$. For completeness, the relative proof uses
$\|H_x-P_x\|^2\le M^2(2C\lambda\varepsilon^2)
\le s_*^2\varepsilon^2\lambda/64$ and the absolute tangent
correction bound $\lambda/4$; these are precisely the
relative-envelope premises. In particular they apply to each
global minimizer.

\paragraph{The indefinite point for the same center and label.}
We use the geometric-persistence construction with
the selected $y=y(v)$, $m=\sigma_n(DG(y))$, and $B=B_y$ to make
clear that it uses the \emph{same} $d,Y,\lambda$. Since $p>n$ and
$B>0$, choose a unit $t\in\ker DG(y)$ with
$b=\mathrm{II}_y(t,t)$ and $\|b\|=B$. The norm is attained on a
diagonal tangent because a scalar symmetric bilinear form has the
same bilinear and diagonal operator norms, after dualizing the
vector output. Put
\[
w=\frac{2b}{B^2},\qquad
r_*=\lambda(DG(y)^\dagger)^Tw,
\qquad q=\frac{4\lambda}{m^2B},\qquad D=Lm+8M^2.
\]
Then $DG(y)^Tr_*=\lambda w$, $\|r_*\|\le2\lambda/(mB)$,
and the tangent identity
$D^2G(y)[t,t]=-DG(y)b$ yields
\begin{equation}\label{eq:robust-force}
\langle r_*,D^2G(y)[t,t]\rangle=-2\lambda.
\end{equation}
Because $m\ge m_0$, $B\ge b_0$, and
$m\mapsto m^4/(Lm+8M^2)$ is increasing, the six bounds in
\eqref{eq:robust-window} imply
\begin{align}\label{eq:robust-budgets}
\frac{Mq}{m}&\le\frac1{32},
&q&\le\frac{m^2B}{16D},
&Bq&\le\frac1{16},\\
\frac{2\lambda}{m^2B^2}&\le\frac{C\varepsilon^2}{8},
&q&\le\frac{\sqrt C\varepsilon}{4},
&q&\le\frac R4.\nonumber
\end{align}
On the $q$-ball in $(\ker DG(y))^\perp$, the map
\[
h\longmapsto DG(y)^\dagger
\bigl(r_*-G(y+h)+G(y)+DG(y)h\bigr)
\]
has center norm at most $q/2$ and Lipschitz constant at most
$Mq/m\le1/32$. It has a fixed point $h$ with $\|h\|<q$ and
$G(y+h)-G(y)=r_*$ exactly. Set $x_\lambda=y+h$. Thus
$G(x_\lambda)-Y=r_*$, and
$\|x_\lambda-x_0\|<R/4+R/4=R/2$.
Also $\|DG(x_\lambda)-DG(y)\|\le Mq\le m/32$, so
$\sigma_n(DG(x_\lambda))\ge31m/32$.

The exact residual, \eqref{eq:robust-anchor-distance}, and
\eqref{eq:robust-budgets} give
\[
\frac{Q(x_\lambda)}\lambda
\le\frac{2\lambda}{m^2B^2}
  +\frac12\bigl(q+\|y-d\|\bigr)^2
\le\left(\frac18+\frac9{128}\right)C\varepsilon^2
=\frac{25}{128}C\varepsilon^2<C\varepsilon^2.
\]
Hence $x_\lambda$ belongs to the \emph{same} low-cost set
that contains the global minimizer.

For its gradient, subtract the exact anchor force:
\[
\nabla Q(x_\lambda)-\lambda w
=\bigl(DG(x_\lambda)-DG(y)\bigr)^Tr_*
 +\lambda(x_\lambda-d).
\]
The elementary $B_0\le M/\sigma$ and persistence give
$B\|y-d\|\le(2M/(s\varepsilon))\delta\varepsilon\le1/8$.
Therefore
\begin{equation}\label{eq:robust-gradient-anchor}
\frac B\lambda
\|\nabla Q(x_\lambda)-\lambda w\|
\le\frac{2Mq}{m}+Bq+B\|y-d\|
\le\frac1{16}+\frac1{16}+\frac18=\frac14,
\qquad
\frac74\le\frac{B\|\nabla Q(x_\lambda)\|}{\lambda}
\le\frac94.
\end{equation}

To check the \emph{complete} Hessian signs, let $t_x$ be the
normalized projection of $t$ onto $\ker DG(x_\lambda)$.
The kernel perturbation bound yields
$\|t_x-t\|\le4Mq/m$ and hence
\[
\|D^2G(x_\lambda)[t_x,t_x]-D^2G(y)[t,t]\|
\le(L+8M^2/m)q=Dq/m.
\]
Equations \eqref{eq:robust-force}--\eqref{eq:robust-budgets} make its
tangent quadratic form at most
\[
t_x^TH_{x_\lambda}t_x
\le-\lambda+\frac{2\lambda Dq}{m^2B}
\le-\frac78\lambda<0.
\]
For any unit $n_x\in(\ker DG(x_\lambda))^\perp$,
$\|DG(x_\lambda)n_x\|\ge31m/32$. The first cap term gives
$M\|r_*\|\le m^2/64$, whence
\[
n_x^TH_{x_\lambda}n_x
\ge(31m/32)^2-m^2/64+\lambda>0.
\]
Both calculations include the residual-weighted second derivatives.

Finally, the full-row-rank pseudoinverse and kernel-projection
perturbation bounds at Jacobian distance $Mq\le m/32$ give
$\|DG(x_\lambda)^\dagger-DG(y)^\dagger\|\le3Mq/m^2$
and matched unit tangents within $4Mq/m$. Applying these in both
directions to the curvature norm, with
$\|DG(x_\lambda)^\dagger\|\le2/m$, yields
\[
|B_{x_\lambda}-B|
\le\left(\frac{2L}{m}+\frac{19M^2}{m^2}\right)q
\le\frac{2Lm+19M^2}{16(Lm+8M^2)}B
\le\frac{19}{128}B<\frac14B.
\]
The last ratio is maximized at $Lm=0$.
Thus $3B/4\le B_{x_\lambda}\le5B/4$. Multiplying these
\emph{anchor-coupled} bounds by
\eqref{eq:robust-gradient-anchor} gives
$21/16\le A(x_\lambda)\le45/16$.
Converting each estimate separately through
$B_0/2\le B\le2B_0$ gives the two stated $B_0$-relative
intervals. This completes the common-domain assertions.
\qed

\subsection{A near-flat family}\label{app:nearflat-example}
The assumptions have a near-flat example. For
$G_\kappa(x,y)=x^2+\kappa y^2$ at $x_0=(\varepsilon,0)$, take
$0<\kappa\le1/4$, $s=M=2$, $L=0$, $r=1/2$,
$R=\varepsilon/2$, and $C=1/512$. Then
$\alpha=\beta=1/(256\sqrt2)$ for every $\kappa$,
$\sigma B_0=2\kappa\to0$, and
$W=\kappa^2\varepsilon^2/32768$. The curvature premise holds
along $y(v)=(\sqrt{\varepsilon^2+v},0)$. Taking $u=0$ and
$v=\beta\varepsilon^2/2$ gives $Q(d)>0$. Thus the passing
minimizer in the theorem is not supplied by an exact center fit.
The ridge window shrinks as this family becomes flatter.

\paragraph{A nonempty near-flat family and a nonexact center.}
For $G_\kappa(x,y)=x^2+\kappa y^2$ at
$x_0=(\varepsilon,0)$, direct differentiation gives
$\sigma=2\varepsilon$, $B_0=\kappa/\varepsilon$,
$\|D^2G_\kappa\|=2$, and $L=0$. With the constants in the
claim, $s_*=1$, $C_*(1,2)=1/512$, and
\[
\delta=\frac1{128\sqrt2},\qquad
\alpha=\beta=\frac1{256\sqrt2},\qquad
\alpha+\sqrt{2C}<\frac12=R/\varepsilon.
\]
The normal section is
$y(v)=(\sqrt{\varepsilon^2+v},0)$ and
$B_{y(v)}/B_0=(1+v/\varepsilon^2)^{-1/2}\in[1/2,2]$
for $|v|\le\beta\varepsilon^2$. The six entries of
$\Lambda/\varepsilon^2$ are, in order,
\[
\frac\kappa{512},\quad
\frac{\kappa^2}{8192},\quad
\frac1{64},\quad
\frac{\kappa^2}{32768},\quad
\frac{\kappa}{32\sqrt{512}},\quad
\frac\kappa{64}.
\]
For $0<\kappa\le1/4$ their minimum, also below one, is
$\kappa^2/32768$. Thus $W$ has the claimed value, while
$\sigma B_0=2\kappa\downarrow0$. The admissible choice
$u=0$, $v=\beta\varepsilon^2/2$ has
$Q(d)=v^2/2=\beta^2\varepsilon^4/8>0$.
\qed
\subsection{Two boundaries of curvature persistence}\label{app:geometry-boundaries}

The next two constructions have separate scopes. The first shows why
the residual direction may need to change with the label. The second
shows why common smoothness bounds alone cannot give a fixed normalized
label radius when curvature vanishes.

\begin{proposition}[Rotating curvature]\label{prop:rotating}
Fix $C,r>0$ and a tube-radius parameter $\rho>0$. Set
\[
 b=\min\{1/16,\sqrt C/16,r/8,\rho/8\},
 \quad k=\pi/b,\quad a(s)=(\cos ks,\sin ks).
\]
For $0<\varepsilon\le1$ let
\[
 G_\varepsilon(t,z)=\varepsilon z+
 \frac{\varepsilon^2t^2}{2}a(z_2/\varepsilon),
 \quad z\in\R^2,\quad x_0=0,\quad R=r\varepsilon.
\]
These maps have common local $C^{2,1}$ bounds. On the exact-fit
section $t=0$, $DG_\varepsilon=(0,\varepsilon I_2)$ and
$B=\varepsilon$ for every $z$, although the curvature vector
rotates by $\pi$ between $z_2=0$ and $z_2=b\varepsilon$.
For every $\norm u\le b\varepsilon$, $\norm v\le b\varepsilon^2$,
and
\[
 0<\lambda\le\min\{C\varepsilon^6/16,
                       \varepsilon^4/(4k)\},
\]
the objective with $d=u$ and $Y=v$ has a strict low-cost witness in
$B(0,\min\{r,\rho\}\varepsilon)$ with
\[
 \frac{3\lambda}{2B_0}\le\norm{\nabla Q}\le
 \frac{5\lambda}{2B_0},\quad B=B_0=\varepsilon,
\]
and complete Hessian eigenvalues
$\{-\lambda,\varepsilon^2+\lambda,
\varepsilon^2+\lambda\}$.
At the allowed label $v=(0,b\varepsilon^2)$, however, the fixed
base residual $r_0=(-2\lambda/\varepsilon^2,0)$ gives a strict
low-cost section point with positive Hessian eigenvalues
$\{3\lambda,\varepsilon^2+\lambda,
\varepsilon^2+\lambda\}$.
\end{proposition}

\begin{proof}
Write $F(s)=s_t^2a(s_2)/2$. The nonlinear part of $G_\varepsilon$
equals $\varepsilon^4F(x/\varepsilon)$, so its second and third
derivatives on $B(0,(r+1)\varepsilon)$ are bounded by
$\varepsilon^2K_2$ and $\varepsilon K_3$, where
$K_j=\sup_{\norm s\le r+1}\norm{D^jF(s)}<\infty$.
Taking $M=\max\{1,K_2\}$ and $L=\max\{1,K_3\}$ proves the
uniform local bounds. At $t=0$ the only nonzero second derivative
is $D^2G[e_t,e_t]=\varepsilon^2a(z_2/\varepsilon)$.
Thus $\II(e_t,e_t)=(0,-\varepsilon a(z_2/\varepsilon))$ and
$B=\varepsilon$.

Put $q=2\lambda/\varepsilon^3$ and $y=v/\varepsilon$. The map
$T(s)=y_2-q\sin(ks/\varepsilon)$ is a contraction on $\R$ since
$qk/\varepsilon\le1/2$. Let $z_2$ be its fixed point and set
$z_1=y_1-q\cos(kz_2/\varepsilon)$. Then
$x=(0,z)$ has the exact residual
\[
 G_\varepsilon(x)-v=-\frac{2\lambda}{\varepsilon^2}
 a(z_2/\varepsilon).
\]
The ridge cap gives $q\le\varepsilon/(2k)<b\varepsilon$,
so $\norm x<2b\varepsilon<\min\{r,\rho\}\varepsilon$
and $\norm{x-u}<3b\varepsilon$. Hence
\[
 \frac{Q(x)}{\lambda\varepsilon^2}
 <\frac C8+\frac{9b^2}{2}
 \le\left(\frac18+\frac9{512}\right)C<C.
\]
At a section point the normal gradient term has norm
$2\lambda/\varepsilon$; the ridge-gradient error is below
$3b\lambda\varepsilon$. This gives the displayed gradient
bounds. The residual-weighted Hessian correction acts only in the
$t$ direction and equals $-2\lambda$ there. Adding the
Gauss--Newton and ridge terms gives the three eigenvalues.

For the last assertion take $v=(0,b\varepsilon^2)$ and
$x_F=(0,-q,b\varepsilon)$. Here
$a(b)=(-1,0)$, $G_\varepsilon(x_F)-v=r_0$, and the tangent
residual correction is $+2\lambda$. The same cost estimate, with
$u=0$, proves strict low cost, and the displayed positive
eigenvalues follow. Thus constant curvature magnitude does not
preserve the sign generated by a fixed base residual.
\end{proof}

\begin{proposition}[A compact-bump boundary]\label{prop:bump}
Let $\chi(q)=(1-q^2)^4$ for $|q|<1$ and zero otherwise, and let
$\phi(t,z)=t^2\chi(t)\chi(z)/2$. Fix $C,r>0$ and a
tube-radius parameter $\rho>0$.
For $0<\varepsilon\le1$, define
\[
 G_\varepsilon(t,z)=\varepsilon z+
 \varepsilon^6\phi(t/\varepsilon^2,z/\varepsilon^2),
 \quad x_0=d=0,\quad R=r\varepsilon.
\]
These maps have fixed global $C^{2,1}$ bounds,
$\sigma=\varepsilon$ and $B_0=\varepsilon$. For every
$\beta>0$ and all sufficiently small $\varepsilon$, the label
$Y=3\varepsilon^3$ lies in
$|Y-G_\varepsilon(x_0)|\le\beta\varepsilon^2$.
For every $0<\lambda\le\varepsilon^4/(8C)$, the low-cost set
is nonempty, but its complete Hessian is positive definite at
every point. Thus no fixed positive normalized label radius follows
from the original common smoothness bounds alone, even if the
positive ridge cap is decreased.
\end{proposition}

\begin{proof}
The cutoff has three continuous derivatives at $\pm1$, so
$\phi\in C^3$ with bounded second and third derivatives.
The chain rule gives
$D^2G_\varepsilon=\varepsilon^2D^2\phi(x/\varepsilon^2)$
and $D^3G_\varepsilon=D^3\phi(x/\varepsilon^2)$.
Thus one can take common positive bounds
$M\ge\sup\norm{D^2\phi}$ and
$L\ge\sup\norm{D^3\phi}$. At the origin,
$DG=(0,\varepsilon)$ and
$D^2G[e_t,e_t]=\varepsilon^2$, giving
$\sigma=B_0=\varepsilon$.

On the square $S_\varepsilon=[-\varepsilon^2,\varepsilon^2]^2$
that contains the curvature support,
$|G_\varepsilon|\le\varepsilon^3+\varepsilon^6/2
\le3\varepsilon^3/2$, hence
$|G_\varepsilon-Y|\ge3\varepsilon^3/2$.
Every low-cost point for the stated ridge range instead has
\[
 |G_\varepsilon-Y|\le\sqrt{2C\lambda}\,\varepsilon
 \le\varepsilon^3/2.
\]
The low-cost set avoids the whole square. Outside it,
$G_\varepsilon(t,z)=\varepsilon z$ locally, so its complete Hessian
is $\operatorname{diag}(\lambda,\varepsilon^2+\lambda)\succ0$.
The exact fit $(0,3\varepsilon^2)$ has
$Q=9\lambda\varepsilon^4/2<C\lambda\varepsilon^2$
when $\varepsilon<\sqrt C/3$, proving nonemptiness.
Also $3\varepsilon^3\le\beta\varepsilon^2$ when
$\varepsilon\le\beta/3$. Every proposed positive ridge cap
contains some $\lambda$ in the stated range, so shrinking it
cannot restore an all-ridges indefinite-witness conclusion.
\end{proof}

\section{Prior certificate comparisons}\label{app:comparison-topic}

\subsection{Squared-distance curvature scale}\label{app:prior-distance}

This calculation follows from the nearest-projection formulas of
Leobacher and Steinicke~\cite[Theorem C and the derivative of
squared distance on p.~16]{leobacher}. Let $\mathcal M$ be a smooth
manifold with a single-valued differentiable nearest-point map $p$
in a normal neighborhood. For
$g(x)=\operatorname{dist}(x,\mathcal M)^2/2$,
$\nabla g(x)=x-p(x)$ and $\nabla^2g(x)=I-Dp(x)$.
At $x=y+t\nu$ with $y\in\mathcal M$ and a unit normal $\nu$,
suppose a principal tangent direction has signed curvature
$\kappa>0$ under the source's shape-operator convention.
Theorem C gives the projection derivative eigenvalue
$(1-t\kappa)^{-1}$ on that direction. Thus
\[
 \lambda_{\rm tangent}(\nabla^2g(x))
 =1-\frac1{1-t\kappa}
 =-\frac{t\kappa}{1-t\kappa}.
\]
If the ridge center is $y$, the objective
$g(x)+\lambda\norm{x-y}^2/2$ has tangent eigenvalue
$\lambda-t\kappa/(1-t\kappa)$ and gradient magnitude
$(1+\lambda)t$. For example, with $0<\lambda<1/4$ and
$t=2\lambda/\kappa$ inside the projection neighborhood, the
tangent eigenvalue is negative and the gradient has order
$\lambda/\kappa$. This is an inference for squared distance to
a fixed manifold. It does not establish a theorem for an arbitrary
raw residual $\norm{G-Y}^2/2$ or for varying center and label.
\subsection{Localized distance-gradient comparison}\label{app:distancegradient}

Let $y\in S$ be a nearest point to a localized $x$, and set $u=x-y$.
For a smooth embedded $S$, $u$ is normal at $y$.
Since $\nabla g(y)=0$, integration along the normal segment gives
\[
 \langle\nabla g(x),u\rangle
 =\int_0^1u^T\nabla^2g(y+tu)u\,dt
 \ge (c/2)\norm u^2.
\]
Cauchy--Schwarz and
$\nabla g=\nabla Q-\lambda(x-d)$ give
(\ref{eq:distancegradient}); the case $u=0$ is immediate.
One localization route under the additional global growth premise
of Masiha et al.~\cite[Lemma E.7, (16)]{masiha} is
$\sqrt{2C\lambda\varepsilon^2/\mu_{QG}}\le r_{\rm loc}$.
If $d\in S$, another is
$\sqrt{2C}\varepsilon\le r_{\rm loc}$, because
$\operatorname{dist}(x,S)\le\norm{x-d}$.
Neither route follows from the gradient inequality itself.
Under the source's tube hypotheses, (\ref{eq:distancegradient})
places $x$ in Lemma E.1's literal radius if
\[
 \norm{\nabla Q(x)}+\lambda\sqrt{2C}\varepsilon
 \le\frac c2\min\{\rho,\lambda/(2\max\{L_{g,3},1\})\}.
\]
For $L_{g,3}>0$, the proof of that lemma compares Hessians by
$L_{g,3}\operatorname{dist}(x,S)$. The same proof therefore
permits $L_{g,3}$ in place of $\max\{L_{g,3},1\}$ in this
condition. This refinement is derived from the proof, not part
of its displayed lemma statement.
\subsection{Comparison of sufficient certificates}\label{app:certificate-comparison}
Table~\ref{tab:certificate-routes} compares three sufficient tests at
the point in Proposition~\ref{prop:negative-separator}. All quantities in the
table use the raw objective~\eqref{eq:objective}. In the averaged
notation of Kharel et al.~\cite[Proposition 5]{kharel},
$\widehat\Delta=E/n$ and $\lambda_{\rm mean}=\lambda/n$.
Thus their margin condition is exactly $\norm E<\lambda$ after
rescaling. We compare this matrix condition, not their generalization
theorem.

\begin{table*}[t]
\centering
\caption{Comparison of specified sufficient conditions. Failure of a
row's test is not failure of every method in the cited work. The
example also admits a directional cost-only check if its observed
cost is used; it does not show that gradient information is necessary.}
\label{tab:certificate-routes}
\setlength{\tabcolsep}{4pt}
\renewcommand{\arraystretch}{1.18}
\begin{tabular}{>{\raggedright\arraybackslash}p{.17\textwidth}>{\raggedright\arraybackslash}p{.38\textwidth}>{\raggedright\arraybackslash}p{.19\textwidth}>{\raggedright\arraybackslash}p{.16\textwidth}}
\hline
Sufficient test & Required information & Guarantee & At Proposition~\ref{prop:negative-separator}'s point\\
\hline
Full residual norm\newline
\cite[Proposition 5]{kharel}
& $\norm E<\lambda$, with ridge and residual correction in the same
normalization.
& Positive full Hessian with margin $\lambda-\norm E$.
& Fails: $\norm E/\lambda=1/(16a)\ge4$.\\
Distance tube\newline
\cite[Lemma E.1]{masiha}
& Source Assumptions 1 and 3: a compact minimizer manifold, local PL
regularity, negative-definite Hessians at exterior critical points,
and a bounded third derivative on a tube. The point must lie within
the stated or proof-refined radius.
& $H\succeq\lambda I/2$ inside that sufficient tube.
& Fails both radii: $\operatorname{dist}/r_{\rm refined}>3/(8a)$.\\
Pointwise directional test\newline
Theorem~\ref{thm:point}
& Cost, Jacobian rank, second-derivative bound, and the tangent
curvature gate~\eqref{eq:gate}.
& $H\succeq\lambda I/2$ and the relative bound in
Corollary~\ref{cor:relative-gate}.
& Passes with $C=1/512$, $s=1$, $M=2$.\\
\hline
\end{tabular}
\end{table*}

\subsection{Proof of the negative-residual comparison}\label{app:negative-separator}

Use the parameters of Proposition~\ref{prop:negative-separator} and put
\[
 t=\frac{a\varepsilon^2}{32}\le2^{-21},\qquad
 h=\varepsilon-X
   =\frac{\lambda}{32a(\varepsilon+X)}.
\]
Then $X=\varepsilon\sqrt{1-t}>\varepsilon/2$ and
$0<h\le a\varepsilon^3/32$.

\paragraph{Cost, rank, and the gradient test.}
Direct differentiation at $z$ gives
\begin{align*}
 r&=-\frac{\lambda}{32a},& J&=(2X,0),&
 \norm{D^2G_a}&=2,\\
 K&=2a,& B&=\frac aX,&
 b&=\left(-\lambda\left[\frac{X}{16a}+h\right],0\right).
\end{align*}
The normalized cost satisfies
\[
 \frac{Q(z)}{\lambda\varepsilon^2}
 =\frac{\varepsilon^2}{2048}+\frac{h^2}{2\varepsilon^2}
 \le\frac{\varepsilon^2(1+a^2\varepsilon^2)}{2048}
 <\frac1{512}.
\]
Also $\sigma_1(J)=2X>\varepsilon$, $\lambda\le\varepsilon^2$,
and $C_*(1,2)=1/512$. The gradient obeys
\[
 \frac{B\norm b}{\lambda}
 =\frac1{16}+\frac{ah}{X}
 \le\frac1{16}+\frac{a^2\varepsilon^2}{16}<\frac18.
\]
Thus the gradient alternative in Theorem~\ref{thm:point} applies
with the stated common constants. One can also check its full gate:
since $\sqrt{2C}=1/16$ and $\varepsilon/X<2$,
\[
 \frac{B(\norm b+\lambda\sqrt{2C}\,\varepsilon)}{\lambda}
 =\frac1{16}+\frac{ah}{X}+\frac{a\varepsilon}{16X}
 <\frac1{16}+\frac{a^2\varepsilon^2}{16}+\frac a8<\frac14.
\]

\paragraph{The complete Hessian.}
The residual correction and full Hessian are
\[
 E=-\frac{\lambda}{16a}\operatorname{diag}(1,a),\qquad
 H=\operatorname{diag}\left(
 4\varepsilon^2-\frac{3\lambda}{16a}+\lambda,
 \frac{15\lambda}{16}\right).
\]
The first eigenvalue is greater than $3\varepsilon^2$, and the
second is $15\lambda/16$. Hence $H\succeq\lambda I/2$.
Although $E\prec0$, its tangent entry is only $-\lambda/16$;
the larger negative entry lies in the normal direction.
There is no mixed block in this example.

\paragraph{Distance to the ellipse.}
For any $(u,v)\in S$, $u\in[-\varepsilon,\varepsilon]$ and
$v^2=(\varepsilon^2-u^2)/a$. Consequently
\[
 \norm{(X,0)-(u,v)}^2-h^2
 =2X(\varepsilon-u)
   +(a^{-1}-1)(\varepsilon^2-u^2)\ge0.
\]
Equality requires $u=\varepsilon$ and $v=0$. Thus $d$ is the
unique nearest point even though $z$ lies inside the ellipse, and
$\operatorname{dist}(z,S)=h$.

\paragraph{The source hypotheses and the two tube radii.}
For each fixed positive $a,\varepsilon$, the set $S$ is a compact,
connected smooth ellipse. The only critical points of
$g=(x^2+ay^2-\varepsilon^2)^2/2$ are $S$ and the origin.
At the origin,
$\nabla^2g=-2\varepsilon^2\operatorname{diag}(1,a)\prec0$,
so $S$ is the only component of local minima. In the shell
$G_a\in[\varepsilon^2/2,3\varepsilon^2/2]$,
\[
 \norm{\nabla G_a}^2=4(x^2+a^2y^2)
 \ge4a(x^2+ay^2)\ge2a\varepsilon^2,
 \qquad
 \norm{\nabla g}^2=2g\norm{\nabla G_a}^2
 \ge4a\varepsilon^2g.
\]
These facts give the local Polyak--Lojasiewicz and minimizer-set
conditions in Assumption 1 of Masiha et al.~\cite{masiha} for a
singleton parameter set. Polynomial smoothness and compactness give
finite third-derivative bounds on each fixed finite tube, as required
by their Assumption 3. Their Lemma E.1 uses these two assumptions.
The Hessian of the quartic $g$ is unbounded globally, so their
global smoothness Assumption 4 does not hold. We do not invoke
their full Lemma E.7.

On $S$, the normal eigenvalue of $\nabla^2g$ is
$\norm{\nabla G_a}^2$. Its minimum is $4a\varepsilon^2$, attained
at $(0,\pm\varepsilon/\sqrt a)$. Since $d\in S$, every admissible
third-derivative bound on a tube containing $S$ satisfies
\[
 L_{g,3}\ge|\partial_x^3g(d)|=12\varepsilon.
\]
The proof of Masiha et al.~\cite[Lemma E.1]{masiha} bounds the
Hessian change by $L_{g,3}\operatorname{dist}(z,S)$. For
$L_{g,3}>0$, it therefore permits the refined radius stated above.
For every admissible choice of the tube constants,
\[
 r_{\rm refined}\le\frac{\lambda}{24\varepsilon},\qquad
 \frac{h}{r_{\rm refined}}
 \ge\frac{3\varepsilon}{4a(\varepsilon+X)}
 >\frac3{8a}.
\]
The literal radius is at most $\lambda/2$, and
\[
 \frac{h}{r_{\rm literal}}
 \ge\frac1{16a(\varepsilon+X)}
 >\frac1{32a\varepsilon}>1.
\]
This proves the two distance claims. For completeness, the normal
Hessian along the nearest-point segment also has a positive lower
bound: for $x\in[X,\varepsilon]$,
\[
 \partial_x^2g(x,0)=6x^2-2\varepsilon^2
 \ge4\varepsilon^2-\frac{3a\varepsilon^4}{16}
 \ge\frac12(4a\varepsilon^2).
\]
This direct check concerns that segment; it is not an application
of the global conclusion of Lemma E.7.

\paragraph{What the comparison shows.}
The absolute residual-norm condition of
Kharel et al.~\cite[Proposition 5]{kharel} fails because
$\norm E/\lambda=1/(16a)\ge4$. Their matrices use sample averages:
for the raw objective, $E=n\widehat\Delta$ and
$\lambda=n\lambda_{\rm mean}$, so the condition becomes exactly
$\norm E<\lambda$. Here $n=1$. We compare this sufficient
matrix condition, not the fitted-point assumptions or generalization
conclusion of that paper.

With the fixed cap $C=1/512$, the cost branch displayed in
Theorem~\ref{thm:point} fails:
\[
 \frac{K\sqrt{2C\lambda}\,\varepsilon}{\lambda/4}
 =\frac1{2\varepsilon}\ge16.
\]
This does not make gradient information necessary. Set
$C_z=Q(z)/(\lambda\varepsilon^2)<C$. A sharper directional test
using this observed cost passes, since
\[
 \left(\frac{K\sqrt{2Q(z)}}{\lambda}\right)^2
 =\frac1{256}+\frac{4a^2h^2}{\lambda}
 \le\frac{1+a^2\varepsilon^2}{256}<\frac1{16}.
\]
All other pointwise assumptions are unchanged when $C$ is replaced
by $C_z$. In particular, this is not a separation from all
cost-based certificates.

As $a\downarrow0$ at fixed $\varepsilon$, the ratios
$\norm E/\lambda$ and $h/r_{\rm refined}$ diverge, but both
$\norm E$ and $h$ tend to zero. The ellipse's long axis
$\varepsilon/\sqrt a$ diverges and its minimum normal gap
$4a\varepsilon^2$ vanishes. The source comparison holds for
each fixed positive pair; it does not assert a uniform compact
family or a uniform positive gap through $a=0$.

\paragraph{An exact relative check.}
This example also permits a direct calculation with
$P=J^TJ+\lambda I$:
\[
 P=\operatorname{diag}(4X^2+\lambda,\lambda),\qquad
 P^{-1/2}EP^{-1/2}
 =\operatorname{diag}\left(-\zeta,-\frac1{16}\right),
 \quad \zeta=\frac{\lambda}{16a(4X^2+\lambda)}<\frac1{16}.
\]
The last inequality follows from
$\lambda<a(4X^2+\lambda)$ under the stated parameter bounds.
Hence the relative norm is exactly $1/16$ and the preconditioned
condition number is $(1-\zeta)/(15/16)<16/15$.
These are pointwise calculations for this example, with no
iteration or runtime conclusion.

\section{Structural bridge and analytic examples}\label{app:structural-topic}

\subsection{Proof of the structural bridge}\label{app:bridge}

All constants denoted by $L_*$ below are finite and depend only on the
fixed reference data and preliminary compact neighborhoods.  They can
increase between estimates, but never depend on $G,\varepsilon,u,\ell$
or $\lambda$. We establish rank from cost before applying the
curvature test.

\paragraph{Fixed regular patch and moving weak coordinates.}
Write $k=n-g$, $m=p-g$.  Choose an injection $S:\R^k\to\R^m$ with
$Dq(z_0)S$ invertible and nested compact neighborhoods $W\Subset W_1$
of $z_0$ on which $\sigma_k(Dq(w))\ge4s_q>0$.
If $g>0$, set $X=A_R^T(A_RA_R^T)^{-1}$, so $A_RX=I$ and $Z^TX=0$.
Let $J_0=DG(c)$, and let $P_J$ be the projector on its bottom $p-g$
right-singular subspace.  For $y$ in this subspace,
$\norm{Ay}\le\kappa(1+\varepsilon)\norm y$.
Since $A$ has a fixed nonzero singular gap on $(\ker A)^\perp$, the
equal-dimensional principal-angle estimate yields
$\norm{P_J-ZZ^T}\le L_*\kappa$.  Put $Z_J=P_JZ$.
For small fixed $\kappa$,
\[
 \norm{Z_J-Z}\le L_*\kappa,\qquad \norm{J_0Z_J}\le\kappa\varepsilon.
\]
Both $[X\ Z_J]$ and its inverse are uniformly bounded, and $J_{0,R}X$
is uniformly invertible.
For $g=0$ set $Z_J=Z$ and omit $X$ and all retained equations.
Define $h_G=G_D-LG_R$ and $K_0=Dh_G(c)$.  The base assumptions give
\begin{equation}\label{eq:K0}
 \norm{K_0Z_J}\le L_*\kappa\varepsilon,
 \qquad \norm{K_0X}\le L_*\kappa.
\end{equation}

\paragraph{Retained solve.}
For $w\in W_1$ put $v=Z_Jw+\varepsilon X\beta$ and
\[
 \Phi_R(\beta,w,\ell_R)=\varepsilon^{-2}
 [G_R(c+\varepsilon v)-G_R(c)]+\ell_R.
\]
On preliminary fixed bounded label sets,
\[
 \Phi_R(0,w,\ell_R)=\varepsilon^{-1}J_{0,R}Z_Jw+\ell_R
 +\int_0^1(1-t)D^2G_R(c+t\varepsilon Z_Jw)[Z_Jw,Z_Jw]\,dt
\]
is uniformly bounded.  Let $P_R=J_{0,R}X$.
On a fixed sufficiently large $\beta$ ball,
$D_\beta\Phi_R=DG_R(c+\varepsilon v)X=P_R+O(M\varepsilon)$.
The map $\beta\mapsto\beta-P_R^{-1}\Phi_R$ equals the uniformly
bounded center $-P_R^{-1}\Phi_R(0,w,\ell_R)$ plus
$O(M\varepsilon(1+\norm\beta))$.  It therefore maps one common closed
ball into itself and is a contraction when $\varepsilon$ is small.
It yields $\beta=\beta(w,\ell_R)$ with uniformly bounded norm.
The ordinary implicit function theorem for each $G$ gives
\[
 D_w\beta=-(D_\beta\Phi_R)^{-1}
                \varepsilon^{-1}DG_R(c+\varepsilon v)Z_J,
 \qquad \norm{D_w\beta}\le L_*.
\]
The last bound uses $J_0Z_J=O(\kappa\varepsilon)$ and the Hessian
magnitude bound.  It uses no modulus for $D^2G$ and no third derivative.

\paragraph{Reduced weak equation and interpolation.}
Substitute the retained solution into
\[
 \Psi(w,\ell)=\varepsilon^{-2}[h_G(c+\varepsilon v(w))-h_G(c)]
                         +\chi(\ell),\quad
 v(w)=Z_Jw+\varepsilon X\beta(w,\ell_R).
\]
Taylor's integral formula,~\eqref{eq:K0}, the full Hessian bound and
\eqref{eq:weakjet} compare this to $\Psi_0(w,\ell)=\chi(\ell)-q(w)$.
In each Hessian integral first replace $v$ by $Z_Jw$, then $Z_J$ by $Z$,
then $\mathcal K_G$ by $\mathcal K_F$ on its two $Z$ arguments, and
finally evaluate the reference tensor at $c$.  This bounds the function
error by
\[
 E_*=L_*\{(1+M)\kappa+M\varepsilon+\omega_F(L_*\varepsilon)\},
\]
where $\omega_F$ is the modulus of the fixed reference Hessian.
The same error scale controls the derivative. For any direction $a$,
\[
 D_w\Psi[a]=\varepsilon^{-1}Dh_G(c+\varepsilon v)Z_Ja
             +Dh_G(c+\varepsilon v)X(D_w\beta[a]).
\]
The first term is $\varepsilon^{-1}K_0Z_Ja+
\int_0^1\mathcal K_G(c+t\varepsilon v)[v,Z_Ja]\,dt$ and admits the
same replacements.  The second is bounded by
$L_*(\kappa+M\varepsilon)\norm a$ using~\eqref{eq:K0}.
Thus $\norm{\Psi-\Psi_0}_{C^1(W_1;w)}\le E_*$, without differentiating
any Hessian. This $C^1$ bound permits a fixed local contraction.
Restricting to $w=z_0+St$ and using a fixed inverse of
$-Dq(z_0)S$ gives a contraction on a sufficiently small fixed ball for
small $E_*$ and $\norm{\ell-\ell_0}$.  Its zero satisfies
$\norm t\le L_*(E_*+\norm{\ell-\ell_0})$.
The resulting exact interpolator obeys
\begin{equation}\label{eq:interpolationbudget}
 \norm{(\theta_{\rm int}-c)/\varepsilon-u_0}
 \le L_*(E_*+R_\ell+\kappa+\varepsilon).
\end{equation}

\paragraph{Every low-cost point, before using its gradient.}
The cost test gives
\[
 \norm{\theta-(c+\varepsilon u)}\le\sqrt{2C}\varepsilon,
 \qquad \norm r\le\sqrt{2C}\varepsilon^2.
\]
Hence $v=(\theta-c)/\varepsilon$ lies within $R_c+\sqrt{2C}$ of $u_0$.
Decompose exactly $v=Z_Jw+X\alpha$.  The uniformly conditioned basis
gives $\norm{w-z_0}\le L_*(R_c+\sqrt C+\kappa)$.
Taylor expansion of the full output
$G(c+\varepsilon v)-G(c)=-\varepsilon^2\ell+r$ gives
\[
 \norm{J_0v}\le\varepsilon
 \bigl(\norm\ell+\sqrt{2C}+(M/2)\norm v^2\bigr)
 \le L_*\varepsilon,
\]
and hence
$\norm\alpha\le L_*\varepsilon$.  Write $\alpha=\varepsilon\beta$;
then $\beta$ is uniformly bounded and $w\in W$ for suitably small
preliminary choices.  This does not use the retained exact solver at the
point, nor any small-gradient condition.

\paragraph{Weak Schur block.}
Apply the fixed row operation $U=\left[\begin{smallmatrix}I&0\\-L&I
\end{smallmatrix}\right]$ and proof-coordinate basis $[X\ Z_J]$ to $J$:
\[
 UJ[X\ Z_J]=\begin{bmatrix}A_s&B_s\\ C_w&D_w\end{bmatrix}.
\]
Here $A_s=J_{0,R}X+O(M\varepsilon)$ is uniformly invertible,
$B_s=O(\varepsilon)$, $C_w=O(\kappa+M\varepsilon)$, and the same
integral replacements give
\[
 \norm{\varepsilon^{-1}D_w+Dq(w)}\le E_*.
\]
Eliminating $C_w$ replaces $D_w$ by
$\mathcal S_w=D_w-C_wA_s^{-1}B_s$.  Its scaled correction is
$O(\kappa+M\varepsilon)$, so
$\sigma_k(\varepsilon^{-1}\mathcal S_w)\ge2s_q$ for small preliminary
parameters, with a fixed upper bound as well.  The row, column and
elimination factors and their inverses are uniformly bounded.  Singular
values are therefore comparable to those of the block diagonal matrix
with an order-one retained block and an order-$\varepsilon$ weak block.
This proves~\eqref{eq:Jbands}, with exact counts $g,n-g$, including $g=0$.

\paragraph{Noncircular selection.}
First fix the regular patches, preliminary compact coordinate/label sets,
and positive upper bounds on $C,R_c,\kappa,\varepsilon_0$ making the
preceding estimates valid uniformly for all smaller choices.  Fix the
Jacobian constants $s,a_s,b_s,b_w$ at this stage.  Next choose $\bar C$
small enough for localization and $\sqrt{2\bar C}<1/4$.
For any prescribed $C\in(0,\bar C]$, subsequently shrink $\kappa$,
choose positive $R_c,R_\ell$, and shrink $\varepsilon_0$ so that
\eqref{eq:interpolationbudget} plus $R_c$ is strictly less than
$\sqrt{2C}/2$, while $R_c+\sqrt{2C}<1/2$ and all points and their
segments to $c$ lie strictly inside $B(c,r_0)$.
This gives $Q(\theta_{\rm int})<C\lambda\varepsilon^2/4$ uniformly
even on the closed center and label balls.  Taking $C$ also below
$C_*(s,M)$ is therefore legitimate before those final shrinkages.
All universal tuple choices occur last.  This completes the proof.

\paragraph{Connection to inverse-function theory.}
The conditions $q(z_0)=\chi(\ell_0)$ and $\rank Dq(z_0)=n-g$
make $(z_0,1)$ a regular zero of $q(z)-t^2\chi(\ell_0)$.
For the augmented reference $\widetilde F(h,t)=F(c+h)+t^2\ell_0$,
this is the regular-zero condition for second-order openness
\cite[Section 5, Theorems 22--24]{sussmann} and local inversion
\cite[Theorem 1]{arutyunov2011}. Interpolation with fixed
$t=\varepsilon$ uses the separate reduced solve above.

Avakov and Magaril-Il'yaev~\cite[Theorem 1]{avakov2021} give uniform
local solvability for continuous maps close to a regular reference.
Arutyunov and Zhukovskiy~\cite[Theorem 1]{arutyunov2025} give
perturbation stability from a regular weighted polynomial truncation.
Here the normalized reference $\Psi_0$ is polynomial and is regular
at $(z_0,\ell_0)$ on the selected slice $w=z_0+St$.
The full $C^1$ estimate above supplies the closeness needed for local
inversion; we use a contraction to obtain the displayed bounds.
The weak-jet and cost estimates also control rank at every low-cost
point, including points that do not interpolate.
\subsection{A quadratic worked family}\label{sec:worked}
The following calculation places the certificate and obstruction in
one model and makes their curvature scales explicit. The flexible
tuple and exact-center tuple have different quantifiers.

\begin{proposition}[Product quadratic]\label{prop:product}
Fix $0<C\le1/1152$, $0<a_0\le1/4$, and
$R_c=R_\ell=\sqrt C/8$. Let
$u_0=(0,1,0)$ and $\ell_0=(0,-1)$. For
$0<\varepsilon\le1/4$, $0\le a\le a_0$,
$0<\lambda\le\varepsilon^2$,
$\norm{u-u_0}\le R_c$, and $\norm{\ell-\ell_0}\le R_\ell$, set
\[
G_a(v,x,y)=(v,x^2+ay^2),\quad d=\varepsilon u,
\quad Y=-\varepsilon^2\ell.
\]
Every whole-space low-cost point lies in
$\overline B(\varepsilon u_0,h\varepsilon)$, where
$h=R_c+\sqrt{2C}<1/4$. At such points the ordered singular
values are $1$ and $2\sqrt{x^2+a^2y^2}\in[\varepsilon,3\varepsilon]$.
Every tuple has an exact interpolator with
$Q< C\lambda\varepsilon^2$ and hence a low-cost global minimizer.

On this tube, $K\le4a$ and $B\le4a/\varepsilon$. If $a>0$, every
low-cost point with $\norm b\le\lambda\varepsilon/(32a)$ obeys
$H\succeq\lambda I/2$; for $a=0$, cost alone suffices. At every such
certified point, the increasing eigenvalues satisfy
\begin{equation}\label{eq:productbands}
\begin{split}
\lambda/2&\le\lambda_1(H)\le5\lambda/4,\\
\varepsilon^2/2&\le\lambda_2(H)\le11\varepsilon^2,\quad
\lambda_3(H)=1+\lambda.
\end{split}
\end{equation}
Every low-cost point is certified by cost alone if
$512Ca^2\varepsilon^2\le\lambda\le\varepsilon^2$.

Now take $a>0$, $u=u_0$, and $\ell=\ell_0$. For \emph{every}
$0<\lambda\le Ca^2\varepsilon^2$, a point $\theta_\lambda$ exists
with $Q(\theta_\lambda)<C\lambda\varepsilon^2$,
$\sqrt3\lambda\varepsilon/a\le\norm{\nabla Q(\theta_\lambda)}
\le2\lambda\varepsilon/a$, and indefinite full Hessian.
Here $Q(d)=0$. The two displayed ridge intervals leave an
unclassified factor-512 gap.
\end{proposition}

The family also meets the weak-jet premise with reference
$F(v,x,y)=(v,x^2)$ at $c=0$ whenever $2a\le\kappa$. Its
indexed ridge and weak bands in (\ref{eq:productbands}) need not be
separated when $\lambda$ is of order $\varepsilon^2$.
The calculations are in Appendix~\ref{app:worked}.

\subsubsection{Proof for the product quadratic}\label{app:worked}

Write $R=\sqrt C/8$ and $h=R+\sqrt{2C}<1/4$.
The cost cap gives $\norm{\theta-d}\le\sqrt{2C}\varepsilon$.
Since $\norm{d-\varepsilon u_0}\le R\varepsilon$, every low-cost
point lies in the stated tube. There $x\ge(1-h)\varepsilon$ and
$|y|\le h\varepsilon$. The Jacobian rows are orthogonal:
\[
 DG_a=\begin{pmatrix}1&0&0\\0&2x&2ay\end{pmatrix}.
\]
Thus its singular values are $1$ and
$2\sqrt{x^2+a^2y^2}$. Since $a\le1/4$ and $h<1/4$,
$\varepsilon\le2\sqrt{x^2+a^2y^2}\le3\varepsilon<1$.

The point
\[
 \theta_{\rm int}=(-\varepsilon^2\ell_1,
                     \varepsilon\sqrt{-\ell_2},0)
\]
matches $Y$ exactly. The label ball ensures $-\ell_2>0$.
The three coordinate differences from $d$ are bounded by
$(5R/4,2R,R)\varepsilon$. Hence
$\norm{\theta_{\rm int}-d}<3R\varepsilon$ and
$Q(\theta_{\rm int})<9C\lambda\varepsilon^2/128$.
For $\lambda>0$, $Q$ is coercive, so a global minimizer exists
and has a value below the cap.

A unit tangent has direction $(0,-ay,x)$. Direct contraction gives
\[
 K=\frac{2a(x^2+ay^2)}{x^2+a^2y^2},\qquad
 B=\frac{a(x^2+ay^2)}{(x^2+a^2y^2)^{3/2}}.
\]
Because $|y/x|<1$ on the tube, $K\le4a$ and
$B\le4a/\varepsilon$. For $a>0$ and
$\norm b\le\lambda\varepsilon/(32a)$,
\[
 B\norm b\le\lambda/8,
 \qquad B\lambda\norm{\theta-d}
 \le4a\sqrt{2C}\lambda<\lambda/8.
\]
Theorem~\ref{thm:point} applies with $s=1$, $M=2$;
$C\le1/1152<C_*(1,2)=1/512$. If $a=0$, then $K=B=0$.
Also $K\sqrt{2C\lambda}\varepsilon\le\lambda/4$ whenever
$\lambda\ge512Ca^2\varepsilon^2$.

The $v$ coordinate is a Hessian block $1+\lambda$.
The $(x,y)$ block has $J_{xy}^TJ_{xy}$ with nonzero eigenvalue
between $\varepsilon^2$ and $9\varepsilon^2$. Since
$\norm E\le2\sqrt{2C\lambda}\varepsilon\le\varepsilon^2/12$
and $\lambda\le\varepsilon^2$, its largest eigenvalue is
less than $11\varepsilon^2<1$, while its larger eigenvalue is
at least $\varepsilon^2-\norm E>\varepsilon^2/2$.
For a certified point, Theorem~\ref{thm:point} gives the lower
bound on its smaller eigenvalue. The tangent Rayleigh quotient
gives its upper bound $\lambda+\lambda/4$.
Thus (\ref{eq:productbands}) holds; in particular, the $v$
eigenvalue is the largest. A global minimizer has $b=0$ and
is certified by the same gate.

For the exact-center tuple let
$\tau=\lambda/(a^2\varepsilon^2)\le C$ and put
\[
 \theta_\lambda=(0,\varepsilon\sqrt{1-a\tau},0).
\]
Then $r=(0,-\lambda/a)$ and
\[
 \frac{Q(\theta_\lambda)}{\lambda\varepsilon^2}
 =\frac\tau2+
 \frac{(1-\sqrt{1-a\tau})^2}{2}<C.
\]
The strict inequality follows from
$1-\sqrt{1-a\tau}\le a\tau$ and $a\le1/4$.
The displacement is less than $a\tau\varepsilon<\varepsilon/8$.
The gradient and full Hessian are
\[
 \norm b=\lambda(2x/a+\varepsilon-x),\qquad
 H=\operatorname{diag}(1+\lambda,
          4\varepsilon^2-6\lambda/a+\lambda,-\lambda).
\]
Since $\sqrt3\varepsilon/2\le x\le\varepsilon$, the gradient
lies in the claimed interval. Also
$H_{xx}\ge(4-6aC)\varepsilon^2\ge(5/2)\varepsilon^2$.
This proves indefiniteness for every ridge in the stated interval.

For
$F(v,x,y)=(v,x^2)$ and $c=0$, retain the first output row.
Then $L=0$, $q(z_x,z_y)=-z_x^2$,
$q(1,0)=\chi(0,-1)=-1$, and $Dq(1,0)$ is onto.
The Hessian discrepancy on the kernel is $2a$, so the bridge
assumptions hold when its chosen tolerance obeys $2a\le\kappa$.

\subsection{Indexed Hessian spectrum}\label{app:indexed-spectrum-proof}
We prove Corollary~\ref{cor:family} using the Jacobian bands and
the pointwise certificate.
\begin{proof}
Min--max on the tangent space gives the upper bound $5\lambda/4$ on its
$p-n$ indices. The normal compression is at least
$s^2\varepsilon^2 I/2$, so interlacing gives the lower weak bound.
Weyl's inequality gives the upper weak bound
$(b_w^2+1+M\sqrt{2C})\varepsilon^2$, and the strong bounds follow from
the fixed strong singular values and $\norm E=O(\varepsilon^2)$.
Empty bands are omitted. Lemma~\ref{lem:bridge} provides the hypotheses
from cost alone. The gradient implication uses the envelope
$B\le M/(s\varepsilon)$. Finally $Q\ge\lambda\norm{\theta-d}^2/2$
is coercive, so an exact strict-cap interpolator supplies a strict-cap
global minimizer. Its gradient is zero and it passes the adaptive test.
\end{proof}

\end{document}